\documentclass[12pt,onecolumn]{IEEEtran}

\usepackage{cite}

\usepackage[dvips]{color}
\usepackage{amssymb}
\usepackage{amsmath}
\usepackage{graphicx}
\usepackage{algorithm}
\usepackage{algorithmic}
\newtheorem{theorem}{Theorem}
\newtheorem{lemma}{Lemma}

\newtheorem{corollary}{Corollary}
\newtheorem{definition}{Definition}
\newtheorem{assumption}{Assumption}

\newcommand{\R}{\mathbb{R}}
\newcommand{\PP}{\mathcal{P}}

\newcommand{\rank}{\mathrm{rank}}
\newcommand{\PI}{\PP_{\mathcal{I}}}
\newcommand{\PIO}{\PP_{\mathcal{I}_0}}

\begin{document}

\title{Robust PCA via Outlier Pursuit}

\author{
Huan
Xu\IEEEmembership{},~Constantine~Caramanis,~\IEEEmembership{Member},~and~
Sujay Sanghavi,~\IEEEmembership{Member}\thanks{The authors are with
the Department of Electrical and Computer Engineering, The
University of Texas at Austin, Austin, TX 78712 USA email:
(huanxu.public@gmail.com; caramanis@mail.utexas.edu;
sanghavi@mail.utexas.edu) } \thanks{A preliminary version appeared in the proceedings of NIPS, 2010 \cite{XuCaramanisSanghavi2010}.}}

\maketitle

\begin{abstract}
Singular Value Decomposition (and Principal Component Analysis)  is
one of the most widely used techniques for dimensionality reduction:
successful and efficiently computable, it is nevertheless plagued by
a well-known, well-documented sensitivity to outliers. Recent work
has considered the setting where each point has a few arbitrarily
corrupted components. Yet, in applications of SVD or PCA such as
robust collaborative filtering or bioinformatics, malicious agents,
defective genes, or simply corrupted or contaminated experiments may
effectively yield entire points that are completely corrupted.

We present an efficient convex optimization-based  algorithm we call
Outlier Pursuit, that under some mild assumptions on the uncorrupted
points (satisfied, e.g., by the standard generative assumption in
PCA problems) recovers the {\it exact} optimal low-dimensional
subspace, and identifies the corrupted points. Such identification
of corrupted points that do not conform to the low-dimensional
approximation, is of paramount interest in bioinformatics and
financial applications, and beyond. Our techniques involve matrix
decomposition using nuclear norm minimization, however, our results,
setup, and approach, necessarily differ considerably from the
existing line of work in matrix completion and matrix decomposition,
since we develop an approach to recover the correct {\it column
space} of the uncorrupted matrix, rather than the exact matrix
itself. In any problem where one seeks to recover a {\it structure} rather than
the {\it exact initial matrices}, techniques developed thus far relying on certificates of optimality, will fail.
We present an important extension of these methods, that allows the treatment of such problems.
\end{abstract}

\section{Introduction}


This paper is about the following problem: suppose we are given a large {\em data matrix} $M$, and we know it can be decomposed as
$$
M = L_0 + C_0,
$$
where $L_0$ is a low-rank matrix, and $C_0$ is non-zero in only a fraction of the columns. Aside from these broad restrictions, both components are arbitrary. In particular we do not know the rank (or the row/column space) of $L_0$, or the number and positions of the non-zero columns of $C_0$. Can we recover the column-space of the low-rank matrix $L_0$, and the identities of the non-zero columns of $C_0$, {\em exactly} and efficiently?

We are primarily motivated by Principal Component Analysis (PCA),
arguably the most widely used technique for dimensionality reduction
in statistical data analysis. The canonical PCA problem
\cite{Jolliffe86}, seeks to find the best (in the least-square-error
sense) low-dimensional subspace approximation to high-dimensional
points. Using the Singular Value Decomposition (SVD), PCA finds the
lower-dimensional approximating subspace by forming a low-rank
approximation to the data matrix, formed by considering each point
as a column; the output of PCA is the (low-dimensional) column space
of this low-rank approximation.

It is well
known~(e.g.,~\cite{Huber81,XuYuille95,ChandrasekaranSanghaviParriloWillsky09,CandesLiMaWright09})
that standard PCA is extremely fragile to the presence of {\em
outliers}: even a single corrupted point can arbitrarily alter the
quality of the approximation. Such non-probabilistic or persistent
data corruption may stem from sensor failures, malicious tampering,
or the simple fact that some of the available data may not conform
to the presumed low-dimensional source / model. In terms of the data
matrix, this means that most of the column vectors will lie in a
low-dimensional space -- and hence the corresponding matrix $L_0$
will be low-rank -- while the remaining columns will be outliers --
corresponding to the column-sparse matrix $C_0$. The natural
question in this setting is to ask if we can still (exactly or
near-exactly) recover the column space of the uncorrupted points,
and the identities of the outliers. This is precisely our problem.

{\em Our results:} We consider a novel but natural convex optimization approach to the recovery problem above. The main result of this paper is to establish that, under certain natural conditions, the optimum of this convex program will yield the column space of $L_0$ and the identities of the outliers (i.e., the non-zero columns of $C_0$). Our conditions depend on the fraction of points that are outliers (which can otherwise be completely arbitrary), and incoherence of the {\em row} space of $L_0$. The latter condition essentially requires that each direction in the column space of $L_0$ be represented in a sufficient number of non-outlier points; we discuss in more detail below. We note that our results do {\em not} require incoherence of the column space, as is done, e.g., in the papers \cite{ChandrasekaranSanghaviParriloWillsky09,CandesLiMaWright09}. This is due to to our alternative convex formulation, and our analytical approach that focuses only on recovery of the column space, instead of ``exact recovery'' of the entire $L_0$ matrix. This also means our method's performance is {\em rotation invariant} -- in particular, applying the same rotation to all given points (i.e., columns) will not change its performance. This is again not true for the method in  \cite{ChandrasekaranSanghaviParriloWillsky09,CandesLiMaWright09}. Finally, we extend our analysis to the noisy case when all points -- outliers or otherwise -- are additionally corrupted by noise.

\subsection*{Related Work}\label{sec:related}

Robust PCA has a long history (e.g.,
\cite{DevlinGnanadesikanKettenring81,XuYuille95,YangWang99,CrousHaesbroeck00,TorreBlack01,TorreBlack03,CrouxFilzmoserOliveira07,Brubaker09}).
Each of these algorithms either performs standard PCA on a robust
estimate of the covariance matrix, or finds directions that maximize a robust
estimate of the variance of the projected data. These algorithms seek to {\it approximately}
recover the column space, and moreover, no existing approach attempts
to identify the set of outliers. This outlier identification, while
outside the scope of traditional PCA algorithms, is important in
a variety of applications such as finance, bio-informatics, and more.

Many existing robust PCA
algorithms suffer two pitfalls: performance degradation with dimension increase, and computational intractability. To wit, \cite{Donoho82} shows that several robust
PCA algorithms including M-estimator \cite{Maronna76}, Convex
Peeling \cite{Barnett76}, Ellipsoidal Peeling
\cite{Titteringto78}, Classical Outlier Rejection
\cite{BarnettLewis78}, Iterative Deletion
\cite{DempsterGasko-Green81} and Iterative Trimming
\cite{DevlinGnanadesikanKettenring75} have breakdown points proportional to the inverse of dimensionality, and hence are useless in the high dimensional regime we consider.

Algorithms with non-diminishing breakdown point, such as
Projection-Pursuit \cite{LiChen85} are non-convex or even
combinatorial, and hence computationally intractable (NP-hard) as
the size of the problem scales. In contrast to these, the
performance of Outlier Pursuit does not depend on the dimension, $p$, and its running time scales gracefully in problem size (in particular, it can be
solved in polynomial time).

Algorithms based on nuclear norm minimization to recover low rank
matrices are now standard, since the seminal paper
\cite{RechtFazelParrilo2010}.
Recent work
\cite{ChandrasekaranSanghaviParriloWillsky09,CandesLiMaWright09} has
taken the nuclear norm minimization approach to the decomposition of
a low-rank matrix and an overall sparse matrix. At a high level,
these papers are close in spirit to ours. However, there are
critical differences in the problem setup,  and the results; for one thing, the algorithms introduced there fail in our setting, as they cannot handle outliers --- entire columns where every entry is corrupted. Beyond this, our approach differs in key analysis techniques, which we believe will prove much more broadly applicable and thus of general interest.

In particular, our work requires a significant extension of existing techniques for matrix decomposition, precisely because the goal is to recover the {\it column space} of $L_0$ (the principal components, in PCA), as opposed to the exact matrices. Indeed, the above works investigate {\em exact} signal recovery --- the intended
outcome is known ahead of time, and one just needs to investigate the conditions needed for success. In our setting, however, the
convex optimization cannot recover $L_0$ itself exactly. We introduce the use of an oracle problem, defined by the structure we seek to recover (here, the true column space). This enables us to show that our convex optimization-based algorithm recovers the correct (or nearly correct, in the presence of noise) column space, as well as the identity of the corrupted points, or outliers.

We believe that this line of analysis will prove to be much more broadly applicable. Often times, exact recovery simply does not make sense under strong corruption models (such as complete column corruption) and the best one can hope for is to capture exactly or approximately, some structural aspect of the problem. In such settings, it may be impossible to follow the proof recipes laid out in works such as \cite{Candes06,CandesLiMaWright09,RechtFazelParrilo2010,ChandrasekaranSanghaviParriloWillsky09}, that essentially obtain exact recovery from their convex optimization formulations. Thus, in addition to our algorithm and our results, we consider the particular proof technique a contribution of potentially general interest.

\section{Problem Setup}

The precise PCA with outlier problem that we consider is as follows:
we are given $n$ points in $p$-dimensional space. A fraction $1-\gamma$ of the points lie on a
 $r$-dimensional {\em true} subspace of the ambient $\mathbb{R}^p$, while the remaining $\gamma n$ points
   are {\em arbitrarily} located -- we call these outliers/corrupted points. We do not have any prior information
   about the true subspace or its dimension $r$. Given the set of points, we would like to learn {\em (a)} the true subspace and {\em (b)} the identities of the outliers.

As is common practice, we collate the points into a $p\times n$ {\em data matrix} $M$, each of whose columns is one of the points, and each of whose rows is one of the $p$ coordinates. It is then clear that the data matrix can be decomposed as
$$
M = L_0 + C_0.
$$
Here $C_0$ is the column-sparse matrix ($(1-\gamma)n$ columns are zero) corresponding to the outliers, and $L_0$ is the matrix corresponding to the non-outliers. Thus, $rank(L_0) = r$, and we assume its columns corresponding to non-zero columns of $C_0$ are identically zero (whatever those columns were cannot possibly be recovered). Consider its Singular Value Decomposition (SVD)
\begin{equation}
L_0 = U_0\Sigma_0 V_0^\top. \label{eq:svd}
\end{equation}
The columns of $U_0$ form an orthonormal basis for the
$r$-dimensional subspace we wish  to recover. 
$C_0$ is the matrix corresponding to the outliers; we
will denote the set of non-zero columns of $C_0$ by $\mathcal{I}_0$,
with $|\mathcal{I}_0|=\gamma n$. These non-zero columns are
completely arbitrary.

With this notation, out intent is to {\em exactly} recover the
column space of $L_0$, and the set of outliers $\mathcal{I}_0$. 
All we are given is the matrix $M$.
Clearly, exact recovery is not always going to be possible (regardless of the
algorithm used) and thus we need to impose a few weak additional
assumptions. We develop these in Section \ref{sec:incoherence}
below.

We are also interested in the noisy case, where
$$
M = L_0 + C_0 + N,
$$
and $N$ corresponds to any additional noise. In this case we are
interested in approximate identification of both the true subspace
and the outliers.

\subsection {Incoherence: When can the column space be recovered ?}\label{sec:incoherence}

In general, our objective of recovering the ``true" column-space of a low-rank matrix that is corrupted with a
column-sparse matrix is not  always a well defined one. As an extreme
example, consider the case where the data matrix $M$ is non-zero in
only one column. Such a matrix is both low-rank and column-sparse, thus the problem is unidentifiable.
To make the problem meaningful, we need to impose that the
low-rank matrix $L_0$ cannot itself be column-sparse as well. This
is done via the following {\em incoherence condition}.

{\bf Definition:} A matrix $L\in \mathbb{R}^{p\times n}$ with SVD $L = U \Sigma V^{\top}$, and $(1-\gamma)n$ of whose columns are non-zero,
is said to be {\em column-incoherent} with parameter $\mu$ if
\[
\max_i \|V^\top \mathbf{e}_i \|^2 \leq \frac{\mu r}{(1-\gamma) n},
\]
where $\{\mathbf{e}_i\}$ are the coordinate unit vectors.

Thus if $V$ has a column aligned with  a coordinate axis, then $\mu
= (1-\gamma)n/r$. Similarly, if $V$ is perfectly incoherent (e.g., if
$r=1$ and every non-zero entry of $V$ has magnitude
$1/\sqrt{(1-\gamma)n}$) then  $\mu = 1$.

In the standard PCA
setup, if the points are generated by some low-dimensional isometric (e.g.,
Gaussian) distribution, then with high probability, one will have
$\mu= O(\max(1, \log (n)/r))$ \cite{CandesRecht09}. Alternatively,
if the points are generated by a uniform distribution over a {\em
bounded} set, then $\mu =\Theta(1)$.

A small incoherence parameter $\mu$
essentially enforces that the matrix $L_0$ will have column support
that is spread out. Note that this is quite natural from the
application perspective. Indeed, if the left hand side is as big as
1, it essentially means that one of the directions of the column
space which we wish to recover, is defined by only a single
observation. Given the regime of a constant fraction of {\it
arbitrarily chosen} and {\it arbitrarily corrupted} points, such a
setting is not meaningful. Having a small incoherence
$\mu$ is an assumption made in all methods based on nuclear norm
minimization up-to-date
\cite{ChandrasekaranSanghaviParriloWillsky09,CandesLiMaWright09,CandesRecht09,CandesTao10}.
Also unidentifiable is the setting where a corrupted point lies in the true subspace. Thus, in matrix terms, we require that every column of $C_0$ does not lie in the column space of $L_0$.

We note that this condition is slightly different from the incoherence conditions required for matrix completion in e.g. \cite{CandesRecht09}. In particular, matrix completion requires row-incoherence (a condition on $U$ of the SVD) and joint-incoherence (a condition on the product $UV$) in addition to the above condition. We do not require these extra conditions because we have a more relaxed objective from our convex program -- namely, we only want to recover the column space.

The parameters $\mu$ and $\gamma$ are not required for the
execution of the algorithm, and {\em do not need to be known a
priori}. They only arise in the analysis of our algorithm's
performance.


{\bf Other Notation and Preliminaries:} Capital letters  such as $A$
are used to represent matrices, and accordingly, $A_i$ denotes the
$i^{th}$ column vector. Letters $U$, $V$, $\mathcal{I}$ and their
variants (complements, subscripts, etc.) are reserved for column
space, row space and column support respectively. There are four
associated projection operators we use throughout. The projection
onto the column space, $U$, is denoted by $\PP_U$ and given by
$\PP_U(A)=U U^\top A$, and similarly for the row-space $\PP_V(A)=A V
V^\top$. The matrix $\PI(A)$ is obtained from $A$ by setting column
$A_i$ to zero for all $i\not\in \mathcal{I}$. Finally, $\PP_T$ is
the projection to the space spanned by $U$ and $V$, and given by
$\PP_T(\cdot)=\PP_{U}(\cdot)+\PP_{V}(\cdot)-\PP_{U}\PP_{V}(\cdot)$.
Note that $\PP_T$ depends on $U$ and $V$, and we suppress this
notation wherever it is clear which $U$ and $V$ we are using. The
complementary operators, $\PP_{U^{\bot}}, \PP_{V^{\bot}}$,
$\PP_{T^{\bot}}$ and $\PP_{{\mathcal I}^c}$ are defined as usual.
The same notation is also used to represent a subspace of matrices:
e.g., we write $A \in \PP_U$ for any matrix $A$ that satisfies
$\PP_U(A)=A$. Five matrix norms are used: $\|A\|_*$ is the nuclear
norm, $\|A\|$ is the spectral norm, $\|A\|_{1,2}$ is the sum of
$\ell_2$ norm of the columns $A_i$, $\|A\|_{\infty, 2}$ is the
largest $\ell_2$ norm of the columns, and $\|A\|_F$ is the Frobenius
norm. The only vector norm used is $\|\cdot\|_2$, the $\ell_2$ norm.
Depending on the context, $I$ is either the unit matrix, or the
identity operator; $\mathbf{e}_i$ is the $i^{th}$ standard basis vector. The
SVD of $L_0$ is $U_0\Sigma_0 V_0$. We use $r$ to denote the rank of $L_0$, and $\gamma\triangleq |\mathcal{I}_0|/n$ the
fraction of outliers.

\section{Main Results and Consequences}

While we do not recover the matrix $L_0$,  we show that the goal of
PCA can be attained: even under our strong corruption model, with a
constant fraction of points corrupted, we show that we can -- under
mild assumptions -- {\em exactly} recover both the column space of
$L_0$ (i.e., the low-dimensional space the uncorrupted points lie on)
and the column support of $C_0$ (i.e. the identities of the
outliers), from $M$. If there is additional noise corrupting the
data matrix, i.e. if we have $M = L_0 + C_0 + N$, a natural variant
of our approach finds a good approximation. In the absence of noise,
an easy post-processing step is in fact able to exactly recover the
original matrix $L_0$. We emphasize, however, that the inability to
do this simply via the convex optimization step, poses significant
technical challenges, as we detail below.

\subsection{Algorithm}


Given the data matrix $M$, our algorithm, called {\it Outlier Pursuit},
generates {\em (a)}  a matrix $U^*$, with orthonormal rows, that
spans the low-dimensional true subspace we want to recover, and {\em
(b)} a set of column indices $\mathcal{I}^*$ corresponding to the
outlier points.

\begin{algorithm*}[h]
Find $(L^*, C^*)$, the optimum of the following convex optimization
program
\begin{equation}\label{equ.convex}
\begin{array}{lcc}
\text{Minimize:} &\quad &   \|L\|_{*} +\lambda \|C\|_{1,2} \\
\text{Subject to:} & &    M=L+C \\
\end{array}
\end{equation}
Compute SVD $L^* = U_1 \Sigma_1 V_1^\top$ and output $U^* = U_1$. \\
Output the set of non-zero columns of $C^*$, i.e. $\mathcal{I}^* =
\{j: c^*_{ij} \neq 0 ~ \text{for some $i$}\}$ \caption{Outlier
Pursuit} \nonumber
\end{algorithm*}


While in the noiseless case there are simple algorithms with similar
performance,  the benefit of the algorithm, and of the analysis, is
extension to more realistic and interesting situations where in
addition to gross corruption of some samples, there is additional
noise. Adapting the Outlier Pursuit algorithm, we have the following variant
for the noisy case.
\begin{equation}\label{equ.noisyOP}
\text{\bf Noisy Outlier Pursuit:} \quad \quad \quad \quad
\begin{array}{lcc}
\text{Minimize:} &\quad &   \|L\|_{*} +\lambda \|C\|_{1,2} \\
\text{Subject to:} & &   \| M - (L+C) \|_F \leq \varepsilon
\end{array}
\end{equation}

Outlier Pursuit (and its noisy variant) is a convex surrogate for
the following natural (but combinatorial and intractable) first
approach to the recovery problem:
\begin{equation}\label{equ.original}
\begin{array}{lcc}
\text{Minimize:} &
\quad &   \rank(L) +\lambda \|C\|_{0,c} \\
\text{Subject to:} & &    M=L+C
\end{array}
\end{equation}
where $\|\cdot\|_{0,c}$ stands for the number of non-zero columns of
a matrix.

\subsection{Performance}

We show that under rather weak assumptions, Outlier Pursuit exactly
recovers the column space of the low-rank matrix $L_0$, and the
identities of the non-zero columns of outlier matrix $C_0$. The
formal statement appears below. \\

\begin{theorem}[Noiseless Case]\label{thm.noiseless}  Suppose we observe $M = L_0 + C_0$,
where $L_0$ has rank $r$ and incoherence parameter $\mu$. Suppose
further that $C_0$ is supported on at most $\gamma n$ columns. Any
output to Outlier Pursuit recovers the column space exactly, and
identifies exactly the indices of columns corresponding to outliers
not lying in the recovered column space, as long as the fraction of
corrupted points, $\gamma$, satisfies
\begin{equation}\label{eq:gamma_mu_condition_1}
\frac{\gamma}{1-\gamma} ~ \leq ~ \frac{c_1}{\mu r},
\end{equation}
where $c_1 = \frac{9}{121}$. This can be achieved by setting the parameter $\lambda$ in
the Outlier Pursuit algorithm to be $\frac{3}{7 \sqrt{\gamma n}}$ -- in fact it holds for
any $\lambda$ in a specific range which we provide below.
\end{theorem}

Note that we only need to know an upper bound on the number of outliers. This is because the success of Outlier Pursuit is monotonic: if it can recover the column space of $L_0$ with a certain set of outliers, it will also recover it when an arbitrary subset of these points are converted to non-outliers (i.e., they are replaced by points in the column space of $L_0$).

For the case where in addition to the corrupted points, we have
noisy observations, $\tilde{M} = M + N$, we have the following
result.
\begin{theorem}[Noisy Case]  \label{thm.noise} Suppose we observe $\tilde{M} = M + N = L_0 + C_0 + N$,
where
\begin{equation}\label{eq:gamma_mu_condition_2}
\frac{\gamma}{1-\gamma} ~ \leq ~ \frac{c_2}{\mu r},
\end{equation}
with $c_2=\frac{9}{1024}$, and $\|N\|_F \leq \varepsilon$. Let the
output of Noisy Outlier Pursuit be $L', C'$. Then there exists
$\tilde{L}, \tilde{C}$ such that $M=\tilde{L}+\tilde{C}$,
$\tilde{L}$ has the correct column space, and $\tilde{C}$ the
correct column support, and
$$
\| L'-\tilde{L}\|_F\leq 10\sqrt{n} \varepsilon;\quad \|
C'-\tilde{C}\|_F\leq 9\sqrt{n} \varepsilon.
$$
\end{theorem}

The  conditions in this theorem are essentially tight in the
following scaling sense (i.e., up to universal constants). If there
is no additional structure imposed beyond what we have stated
above, then up to scaling, in the noiseless case, Outlier Pursuit
can recover from as many outliers (i.e., the same fraction) as any
algorithm of possibly arbitrary complexity. In particular, it is
easy to see that if the rank of the matrix $L_0$ is $r$, and the
fraction of outliers satisfies $\gamma \geq 1/(r+1)$, then the
problem is not identifiable, i.e., no algorithm can separate
authentic and corrupted points. In the presence of stronger assumptions (e.g., isometric
distribution) on the authentic points, better recovery guarantees are possible \cite{HRPCA-COLT}.

\section{Proof of Theorem~\ref{thm.noiseless}}\label{sec.proofnoise}

In this section and the next section, we prove
Theorem~\ref{thm.noiseless} and Theorem~\ref{thm.noise}. Past matrix
recovery papers, including
\cite{ChandrasekaranSanghaviParriloWillsky09,CandesLiMaWright09,CandesRecht09},
sought exact recovery. As such, the generic (and successful) roadmap
for the proof technique was to identify the first-order necessary
and sufficient conditions for a feasible solution to be optimal, and
then show that a subgradient certifying optimality of the desired
solution exists under the given assumptions. In our setting, the
outliers, $C_0$, preclude exact recovery of $L_0$. In fact, the
optimum $\hat{L}$ of (\ref{equ.convex}) will be non-zero in every
column of $C_0$ that is not {\em orthogonal} to $L_0$'s column space
-- that is, Outlier Pursuit (\ref{equ.convex}) cannot recover $L_0$
on the columns corresponding to the outliers (intuitively, no method
can -- there is nothing left to recover once the entire point is
corrupted, and our choice of setting the corresponding columns of $L_0$ to zero is arbitrary). Thus a dual certificate certifying optimality of $(L_0,C_0)$ will not exist, in general. However, all we require for success is to recover a pair $(\hat{L},\hat{C})$
where $\hat{L}$ has the correct column space and $\hat{C}$ the
correct column support. And thus, rather than construct a dual certificate for optimality of $(L_0,C_0)$, all we need is a dual certificate for {\em any pair} $(\hat{L},\hat{C})$ as above. The challenge is that we do not know, {\it a
priori}, what that pair will be, and hence cannot follow the standard road map to write optimality conditions for a specific pair.

The main new ingredient of the proof of correctness and the analysis
of the algorithm,  is the introduction of an oracle problem with
additional side constraints, that produces a solution with the
correct column space and support. Thus, we have the following:
\subsection*{Roadmap of the Proof}
\begin{enumerate}
\item We define an oracle problem, with additional side constraints that enforce the right column space and support.
\item We then write down the properties a dual certificate must satisfy to certify optimality of {\it the solution to the oracle problem}.
\item We construct a dual certificate, thereby obtaining conditions for the range of $\lambda$ for which recovery is guaranteed.
\end{enumerate}

Before going into technical details, we list some technical preliminaries that we use multiple times in the sequel. The following lemma is well-known, and gives the subgradient of the norms we consider. \\

\begin{lemma}\label{lem.subgradients}
For any column space $U$, row
space $V$ and column support $\mathcal{I}$:
\begin{enumerate}
\item Let the SVD of a matrix $A$ be $U \Sigma V^\top$. Then the
subgradient
  to $\|\cdot\|_{*}$ at $A$ is $\{U V^\top+W|
\PP_{T}(W)=0,\,\|W\|\leq 1\}$.
\item Let the column support of a matrix $A$ be $\mathcal{I}$. Then
the subgradient to $\|\cdot\|_{1,2}$ at $A$ is $\{H+Z| \PI(H)=H,
H_i=A_i/\|A_i\|_2;\, \PI(Z)=0,\|Z\|_{\infty,2}\leq 1\}$.
\item For any $A$, $B$, we have $\PI(AB)=A \PI(B)$; for any $A$,
$\PP_U\PI(A)=\PI\PP_U(A)$.
\end{enumerate}
\end{lemma} 

\vspace{0.1in}

\begin{lemma} If a matrix $\tilde{H}$ satisfies $\|\tilde{H}\|_{\infty, 2}\leq 1$ and is supported on
$\mathcal{I}$, then $\|\tilde{H}\|\leq \sqrt{|\mathcal{I}|}$.
\end{lemma}

\vspace{0.05in}

\begin{proof} Using the variational form of the operator norm, we have
\begin{equation*}\begin{split}
\|\tilde{H}\|&= \max_{\|\mathbf{x}\|_2 \leq 1, \|\mathbf{y}\|_2\leq
1} \mathbf{x}^\top \tilde{H}\mathbf{y}\\&=\max_{\|\mathbf{x}\|_2\leq
1 }\|\mathbf{x}^\top \tilde{H} \|_2 =\max_{\|\mathbf{x}\|_2\leq
1}\sqrt{\sum_{i=1}^n (\mathbf{x}^\top \tilde{H}_i)^2} \leq
\sqrt{\sum_{i\in \mathcal{I}} 1}=\sqrt{|\mathcal{I}|}.
\end{split}\end{equation*}
The inequality holds because $\|\tilde{H}_i\|_2=1$ when $i\in
\mathcal{I}$, and equals zero otherwise.\end{proof}

\vspace{0.1in}

\begin{lemma}
Given a matrix $U\in \mathbb{R}^{r\times n}$ with orthonormal columns,
and any matrix $\tilde{V}\in \mathbb{R}^{r\times n}$, we have that $\|U
\tilde{V}^\top\|_{\infty,2} = \max_i \|\tilde{V}^\top
\mathbf{e}_i\|_2.$\end{lemma}

\vspace{0.05in}

\begin{proof}By definition we have
\begin{equation*}\begin{split}
&\|U \tilde{V}^{\top}\|_{\infty,2} =\max_i \| U
\tilde{V}^{\top}_i\|_2 \stackrel{(a)}{=}\max_i
\|\tilde{V}^{\top}_i\|_2 =\max_i\|\tilde{V}^{\top}\mathbf{e}_i\|_2.
\end{split}\end{equation*}
Here (a) holds since $U$ has  orthonormal columns.
\end{proof}

\subsection{Oracle Problem and Optimality Conditions}
As discussed, in general Outlier Pursuit will not recover the true solution $(L_0,C_0)$, and hence it is not possible to construct a subgradient certifying optimality of $(L_0,C_0)$.  Instead, our goal is to recover any pair $(\hat{L},\hat{C})$ so that $\hat{L}$ has the correct column space, and $\hat{C}$ the correct column support. Thus we need only construct a dual certificate for some such pair. We develop our candidate solution $(\hat{L},\hat{C})$ by imposing precisely these constraints on the original optimization problem (\ref{equ.convex}): the solution $\hat{L}$ should have the correct column space, and $\hat{C}$ should have the correct column support. 

Let the SVD of the true $L_0$ be $L_0 = U_0\Sigma_0 V_0^\top$, and recall that the projection of any matrix $X$ onto the space of all matrices with column space contained in $U_0$ is given by $\PP_{U_0}(X) := U_0U_0^\top X$. Similarly for the
column support $\mathcal{I}_0$ of the true $C_0$, the projection $\PP_{\mathcal{I}_0}(X)$ is the matrix that results when
all the columns in $\mathcal{I}_0^c$ are set to 0.

Note that $U_0$ and $\mathcal{I}_0$ above correspond to the {\em truth}. Thus, with this notation, we would like the optimum of (\ref{equ.convex}) to satisfy $\PP_{U_0}(\hat{L})=\hat{L}$, as this is nothing but the fact that $\hat{L}$ has recovered the true subspace. Similarly, having $\hat{C}$ satisfy $\PIO(\hat{C})=\hat{C}$ means that we have succeeded in identifying the outliers. The oracle problem arises by {\em imposing} these as additional constraints in (\ref{equ.convex}):
\begin{equation}\label{equ.side}
\text{\bf Oracle Problem:} \quad \quad \quad \quad
\begin{array}{lcc}
\text{Minimize:} &\quad &   \|L\|_{*} +\lambda \|C\|_{1,2} \\
\text{Subject to:} & &    M=L+C; \,\, \PP_{U_0}(L)=L;\,\,  \PIO(C)=C.
\end{array}
\end{equation}
The problem is of course bounded (by zero), and is feasible, as $(L_0, C_0)$ is a feasible solution. Thus, an optimal solution, denoted as $\hat{L}, \hat{C}$ exists. We now show that the solution $(\hat{L}, \hat{C})$ to the oracle problem, is also an optimal solution to Outlier Pursuit. Unlike the original pair $(L_0,C_0)$, we can certify the optimality of $(\hat{L},\hat{C})$ by constructing the appropriate subgradient witness.

The next lemma and definition, are key to the development of our optimality conditions. \\

\begin{lemma}\label{lem.firstnoneoth} Let the pair $(L', C')$ satsify $L'+C'=M$, $\PP_{U_0}(L')=L'$, and
$\PIO(C')=C'$. Denote the SVD of $L'$ as $L'=U'\Sigma V'^\top$, and the
column support of $C'$ as $\mathcal{I}'$. Then $U' U'^\top=U_0 U_0^\top$,
and $\mathcal{I}'\subseteq \mathcal{I}_0$.
\end{lemma}
\begin{proof} The only thing we need to prove is that $L'$ has a
rank no smaller than $U_0$. However, since $\PIO(C')=C'$, we must have $\PP_{\mathcal{I}_0^c}(L')=\PP_{\mathcal{I}_0^c}(M)$, and thus the rank of $L'$ is at least as large as
$\PP_{\mathcal{I}_0^c}(M)$, hence $L'$ has a rank no smaller than $U_0$.
\end{proof}

Next we define two operators that are closely related to the subgradient
of $\|L'\|_*$ and $\|C'\|_{1,2}$.

\begin{definition} Let $(L', C')$ satisfy $L'+C'=M$, $\PP_{U_0}(L')=L'$, and
$\PIO(C')=C'$.  We define the following:
\begin{equation*}\begin{split}
&\mathfrak{N}(L')\triangleq U'V'^\top;\\
&\mathfrak{G}(C')\triangleq \left\{H\in \mathbb{R}^{m\times
n}\left|\PP_{\mathcal{I}_0^c}(H)=0;\,\,\forall i\in
\mathcal{I}':\,H_i=\frac{C_i'}{\|C_i'\|_2};\,\,\forall i\in
\mathcal{I}_0\cap(\mathcal{I}')^c:\, \|H_i\|_2\leq 1
\right.\right\},
\end{split}\end{equation*}
where the SVD of $L'$ is $L'=U'\Sigma V'^\top$, and the column
support of $C'$ is $\mathcal{I}'$. Further define the operator
$\PP_{T(L')}(\cdot): \mathbb{R}^{m\times n}\rightarrow
\mathbb{R}^{m\times n}$ as
\[\PP_{T(L')}(X)=\PP_{U'}(X)+\PP_{V'}(X)-\PP_{U'}\PP_{V'}(X).\]
\end{definition}

Now we present and prove the optimality condition (to Outlier
Pursuit) for solutions $(L,C)$  that have the correct column space
and support for $L$ and $C$, respectively.
\begin{theorem}\label{thm:dualconditions}
Let $(L', C')$ satisfy $L'+C'=M$, $\PP_{U_0}(L')=L'$, and
$\PIO(C')=C'$. Then $(L', C')$ is an optimal solution of Outlier
Pursuit if there exists a matrix $Q\in \mathbb{R}^{m\times n}$ that satisfies
\begin{equation}\label{equ.dualcertificate}\begin{split}
(a)\quad &\mathcal{P}_{T(L')}(Q) =\mathfrak{N}(L');\\
(b)\quad &\|\mathcal{P}_{T(L')^\bot}(Q)\| \leq
1;\\
(c)\quad&\mathcal{P}_{\mathcal{I}_0}(Q)/\lambda \in
\mathfrak{G}(C');\\(d)\quad&\|\mathcal{P}_{\mathcal{I}_0^{c}}(Q)\|_{\infty,2}\leq
\lambda.
\end{split}\end{equation}
If both inequalities are strict (dubbed {\em $Q$ strictly
satisfies~(\ref{equ.dualcertificate})}), and
$\PP_{\mathcal{I}_0}\cap\PP_{V'}=\{0\}$, then any optimal solution will have the right column space, and column support.
\end{theorem}
\begin{proof} By standard convexity arguments \cite{Rockafellar70}, a feasible pair $(L', C')$ is an optimal solution of
Outlier Pursuit, if there exists a $Q'$ such that
$$
Q' \in \partial \|L'\|_*;\quad Q'\in \lambda \partial \|C'\|_{1,2}.
$$
Note that (a) and (b) imply that $Q \in \partial \|L'\|_*$.
Furthermore, letting $\mathcal{I}'$ be the support of $C'$, then by Lemma \ref{lem.firstnoneoth},
$\mathcal{I}'\subseteq \mathcal{I}_0$. Therefore (c) and (d) imply
that
$$
Q_i=\frac{\lambda C_i'}{\|C_i'\|_2};\quad \forall i\in \mathcal{I}';
$$
and
$$
\|Q_i\|_2\leq \lambda;\quad \forall i\not\in \mathcal{I'},
$$
which implies that $Q\in \lambda \partial \|C'\|_{1,2}$. Thus, $(L', C')$ is an
optimal solution.

The rest of the proof establishes that when (b) and (d) are strict,
then any optimal solution  $(L'', C'')$ satisfies
$\PP_{U_0}(L'')=L''$, and $\PIO(C'')=C''$. We show that for any
fixed $\Delta\not=0$, $(L'+\Delta, C'-\Delta)$ is strictly worse
than $(L', C')$, unless $\Delta\in \PP_{U_0} \cap \PIO$. Let $W$ be
such that $\|W\|=1$, $\langle W,
\PP_{T(L')^\bot}(\Delta)\rangle=\|\PP_{T(L')^\bot} \Delta\|_{*}$,
and $\PP_{T(L')} W=0$. Let $F$ be such that
$$
F_i=\left\{\begin{array}{ll} \frac{-\Delta_i}{\|\Delta_i\|_2} &
\mbox{if}\,\, i\not\in \mathcal{I}_0,\,\, \mbox{and}\,\, \Delta_i \not=0\\
0 & \mbox{otherwise.}\end{array}\right.
$$
Then $\PP_{T(L')}(Q)+W$ is a subgradient of $\|L'\|_*$ and
$\mathcal{P}_{\mathcal{I}_0}(Q)/\lambda+F$ is a subgradient of
$\|C'\|_{1,2}$. Then we have
\begin{equation*}\begin{split}&\|L'+\Delta\|_*+\lambda\|C'-\Delta\|_{1,2}\\
\geq &\|L'\|_*+\lambda\|C'\|_{1,2}+<\PP_{T(L')}(Q)+W, \Delta>
-\lambda <\mathcal{P}_{\mathcal{I}_0}(Q)/\lambda+F,
\Delta>\\
=&\|L'\|_*+\lambda\|C'\|_{1,2}+
\|\mathcal{P}_{{T(L')}^\bot}(\Delta)\|_*+\lambda
\|\mathcal{P}_{\mathcal{I}_0^c}(\Delta)\|_{1,2}+ <\PP_{T(L')}(Q)
-\PIO(Q),
\Delta>\\
=&\|L'\|_*+\lambda\|C'\|_{1,2}+
\|\mathcal{P}_{{T(L')}^\bot}(\Delta)\|_*+\lambda
\|\mathcal{P}_{\mathcal{I}_0^c}(\Delta)\|_{1,2}+
<Q-\mathcal{P}_{{T(L')}^\bot}(Q)
-(Q-\mathcal{P}_{\mathcal{I}_0^c}(Q)), \Delta>
\\
=&\|L'\|_*+\lambda\|C'\|_{1,2}+
\|\mathcal{P}_{{T(L')}^\bot}(\Delta)\|_*+\lambda
\|\mathcal{P}_{\mathcal{I}_0^c}(\Delta)\|_{1,2}+
<-\mathcal{P}_{{T(L')}^\bot}(Q), \Delta>+<
\mathcal{P}_{\mathcal{I}_0^c}(Q),
\Delta>\\
\geq
&\|L'\|_*+\lambda\|C'\|_{1,2}+(1-\|\mathcal{P}_{{T(L')}^\bot}(Q)\|)
\|\mathcal{P}_{{T(L')}^\bot}(\Delta)\|_*+(\lambda-\|\mathcal{P}_{\mathcal{I}_0^c}(Q)\|_{\infty,
2})
\|\mathcal{P}_{\mathcal{I}_0^c}(\Delta)\|_{1,2}\\
\geq & \|L'\|_*+\lambda\|C'\|_{1,2},\end{split}\end{equation*}where
the last inequality is strict unless
\begin{equation}\label{equ.inspacenew}\|\mathcal{P}_{{T(L')}^\bot}(\Delta)\|_*=\|\mathcal{P}_{\mathcal{I}_0^c}(\Delta)\|_{1,2}=0.\end{equation}
Note that (\ref{equ.inspacenew}) implies that
$\PP_{T(L')}(\Delta)=\Delta$ and $\PIO(\Delta)=\Delta$. Furthermore
\[\PIO(\Delta)=\Delta=\PP_{T(L')}(\Delta)=\PP_{U'}(\Delta)+\PP_{V'}\PP_{{U'}^\bot}(\Delta)=\PIO\PP_{U'}(\Delta)+\PP_{V'}\PP_{{U'}^\bot}(\Delta),\]
where the last equality holds because we can write
$\PIO(\Delta)=\Delta$. This leads to
$$
\PIO\PP_{{U'}^\bot}(\Delta)=\PP_{V'}\PP_{{U'}^\bot}(\Delta).
$$
Lemma~\ref{lem.firstnoneoth} implies $\PP_{U'}=\PP_{U_0}$, which
means $\PP_{U_0^\bot}(\Delta)\in \PIO \cap \PP_{V'}$, and hence
equal $0$. Thus, $\Delta\in \PP_{U_0}$. Recall that
Equation~(\ref{equ.inspacenew}) implies $\Delta\in \PIO$, we then
have $\Delta\in \PIO \cap\PP_{U_0}$, which completes the proof.
\end{proof}

Thus, the oracle problem determines a solution pair, $(\hat{L},\hat{C})$, and then using this, Theorem \ref{thm:dualconditions} above, gives the conditions a dual certificate must satisfy. The rest of the proof seeks to build a dual certificate for the pair $(\hat{L},\hat{C})$. To this end, The following two results are quite helpful in what follows. For the remainder of the paper, we use $(\hat{L},\hat{C})$ to denote the dual pair that is the output of the oracle problem, and we assume that the SVD of $\hat{L}$ is given as $\hat{L} = \hat{U} \hat{\Sigma}\hat{V}^{\top}$.
\begin{lemma}\label{le:vbar} There exists an orthonormal matrix $\overline{V}\in \mathbb{R}^{r\times n}$
such that
$$
\hat{U}\hat{V}^\top =U_0 {\overline{V}}^\top.
$$
In addition,
$$
\PP_{\hat{T}}(\cdot)\triangleq \PP_{\hat{U}}
(\cdot)+\PP_{\hat{V}}(\cdot)-\PP_{\hat{U}}\PP_{\hat{V}}(\cdot)
=\PP_{U_0}(\cdot)+\PP_{{\overline{V}}}(\cdot)-\PP_{U_0}\PP_{{\overline{V}}}(\cdot).
$$
\end{lemma}
\begin{proof}Due to Lemma~\ref{lem.firstnoneoth}, we have
$U_0U_0^\top=\hat{U}\hat{U}^\top$, hence $U_0=\hat{U}\hat{U}^\top
U_0$. Letting ${\overline{V}}=\hat{V}\hat{U}^\top U_0$, we
have $\hat{U}\hat{V}^\top =U_0 {\overline{V}}^\top$, and
${\overline{V}}{\overline{V}}^\top=\hat{V}\hat{V}^\top$. Note that
$U_0U_0^\top=\hat{U}\hat{U}^\top$ leads to $\PP_U=\PP_{\hat{U}}$,
and ${\overline{V}}{\overline{V}}^\top=\hat{V}\hat{V}^\top$ leads to
$\PP_{\overline{V}}=\PP_{\hat{V}}$, so the second claim follows.
\end{proof}

Since $\hat{L},\hat{C}$ is an optimal solution to Oracle
Problem~(\ref{equ.side}), there exists $Q_1$,
$Q_2$, $A'$ and $B'$ such that
\[Q_1+\PP_{U_0^\bot}(A')=Q_2+\PP_{\mathcal{I}_0^c}(B'),\] where $Q_1$, $Q_2$ are subgradients to
$\|\hat{L}\|_*$ and to $\lambda \|\hat{C}\|_{1,2}$, respectively.
This means that $Q_1= U_0{\overline{V}}^\top+W$ for some orthonormal
${\overline{V}}$ and $W$ such that $\PP_{\hat{T}}(W)=0$, and
$Q_2=\lambda({\hat{H}}+Z)$ for some ${\hat{H}}\in
\mathfrak{G}(\hat{C})$,
 and $Z$ such that $\PIO(Z)=0$.
Letting $A=W+A'$, $B=\lambda Z+B'$, we have
\begin{equation}\label{equ.uvequalsh}\begin{split}
U_0 {\overline{V}}^\top +\PP_{U_0^\bot}(A)=\lambda {\hat{H}}
+\PP_{\mathcal{I}_0^c}(B).\end{split}\end{equation} Recall that
${\hat{H}}\in \mathfrak{G}(\hat{C})$ means
$\PIO({\hat{H}})={\hat{H}}$ and $\|{\hat{H}}\|_{\infty,2}\leq 1$.
\begin{lemma}\label{lem.uvequalsh}
We have \[U_0\PP_{\mathcal{I}_0}({\overline{V}}^\top) = \lambda \PP_{U_0}
({\hat{H}}).\]
\end{lemma}
\begin{proof}
We have
\begin{eqnarray*}
\PP_{U_0} \PP_{\mathcal{I}_0}(U_0{\overline{V}}^\top+\PP_{U_0^\bot}(A))
&=& \PP_{U_0}\PP_{\mathcal{I}_0}(U_0{\overline{V}}^\top)+\PP_{U_0}\PP_{\mathcal{I}_0}(\PP_{U_0^\bot}(A)) \\
&=& U_0\PP_{\mathcal{I}_0}({\overline{V}}^\top)+\PP_{U_0}\PP_{U_0^\bot}\PP_{\mathcal{I}_0}(A) \\
&=& U_0\PP_{\mathcal{I}_0}({\overline{V}}^\top).
\end{eqnarray*}
Furthermore, we have
$$
\PP_{U_0}\PIO(\lambda {\hat{H}}+\PP_{\mathcal{I}_0^c}(B))=\lambda \PP_{U_0}({\hat{H}}).
$$
The lemma follows from~(\ref{equ.uvequalsh}).
\end{proof}

\subsection{Obtaining Dual Certificates for Outlier Pursuit}

In this section, we complete the proof of Theorem~\ref{thm.noiseless}
by constructing a dual certificate for $(\hat{L}, \hat{C})$ -- the solution to the oracle problem -- showing it is also the solution to Outlier Pursuit.
The conditions the dual certificate must satisfy are spelled out in
Theorem \ref{thm:dualconditions}. It is helpful to first consider the simpler case where the
corrupted columns are assumed to be orthogonal to the column space
of $L_0$ which we seek to recover. Indeed, in that setting, we have
$V_0=\hat{V}=\overline{V}$, and moreover, straightforward algebra
shows that we automatically satisfy the condition $\PIO
\cap\PP_{V_0}=\{0\}$. (In the general case, however, we require an
additional condition to be satisfied, in order to recover the same
property.) Since the columns of $H_0$ are either zero, or defined as normalizations of the columns of matrix $C_0$
(i.e., normalizations of outliers), we immediately conclude that $\PP_{U_0}(H) =
\PP_{V_0}(H) = \PP_T(H) = 0$, and also $\PIO(U_0V_0^{\top}) = 0$. As a result, it is not hard to verify
that the dual certificate for the orthogonal case is:
$$
Q_0 = U_0V_0^{\top} + \lambda H_0.
$$
While not required for the proof of our main results, we include the
proof of the orthogonal case in Appendix \ref{app:orthogonal}, as
there we get a stronger {\em necessary and sufficient} condition for
recovery.

For the general, non-orthogonal case, however, this certificate does
not satisfy the conditions of Theorem \ref{thm:dualconditions}. For
instance, $\PP_{V_0}(H_0)$ need no longer be zero, and hence
the condition $\PP_T(Q_0) = U_0V_0^{\top}$ may no longer hold. We correct
for the effect of the non-orthogonality by modifying $Q_0$ with
matrices $\Delta_1$ and $\Delta_2$, which we define below.

Recalling the definition of $\overline{V}$ from Lemma \ref{le:vbar},
define matrix   $ G\in \mathbb{R}^{r\times r}$ as
$$
G\triangleq \PIO({\overline{V}}^\top) (\PIO({\overline{V}}^\top))^\top.
$$
Then we have
$$
G=\sum_{i\in \mathcal{I}_0} [({\overline{V}}^\top)_i][({\overline{V}}^\top)_i]^\top \preceq
\sum_{i=1}^n [({\overline{V}}^\top)_i][({\overline{V}}^\top)_i]^\top
= {\overline{V}}^\top {\overline{V}} =I,
$$
where $\preceq$ is the generalized inequality induced by the
positive semi-definite cone. Hence, $\|G\| \leq 1$. The following lemma bounds $\|G\|$ away from $1$.

\vspace{0.1in}

\begin{lemma}\label{lemma.c} Let $\psi=\|G\|$. Then $\psi \leq \lambda^2 \gamma n$. In particular, for $\lambda \leq \frac{3}{7 \sqrt{\gamma n}}$, we have $\psi < \frac{1}{4}$.
\end{lemma}

\begin{proof} We have
$$
\psi=\|U_0 \PIO({\overline{V}}^\top)
(\PIO({\overline{V}}^\top))^\top U_0^\top\|=\|[U_0
\PIO({\overline{V}}^\top)] [U_0 \PIO({\overline{V}}^\top)]^\top\|,
$$
due to the fact that $U_0$ is orthonormal. By Lemma~\ref{lem.uvequalsh}, this implies
\begin{eqnarray*}
\psi &=& \|[\lambda \PP_{U_0}({\hat{H}})][\lambda \PP_{U_0}({\hat{H}})]^\top\| \\
&=& \lambda^2 \|\sum_{i\in \mathcal{I}_0} \PP_{U_0}({\hat{H}}_i)\PP_{U_0}({\hat{H}}_i)^\top\| \\
&\leq& \lambda^2 |\mathcal{I}_0| \\
&=& \lambda^2 \gamma n.
\end{eqnarray*}
The inequality holds because $\|\PP_{U_0}({\hat{H}}_i)\|_2\leq 1$ implies $\|\PP_{U_0}({\hat{H}}_i)\PP_{U_0}({\hat{H}}_i)^\top\|\leq 1$. 
\end{proof}

\vspace{0.1in}

\begin{lemma}\label{lem.critical} If $\psi<1$, then the following operation
$\PP_{\overline{V}}\PP_{\mathcal{I}_0^c}\PP_{\overline{V}}$ is an injection from
$\PP_{\overline{V}}$ to $\PP_{\overline{V}}$, and its inverse
operation is $I+\sum_{i=1}^\infty
(\PP_{\overline{V}}\PIO\PP_{\overline{V}})^i$.
\end{lemma}

\begin{proof} Fix matrix $X\in \mathbb{R}^{p\times n}$ such that $\|X\|=1$, we have that
\begin{eqnarray*}
\PP_{\overline{V}}\PIO\PP_{\overline{V}}(X) & = & \PP_{\overline{V}}\PIO(X {\overline{V}} {\overline{V}}^\top) \\
& = & \PP_{\overline{V}}(X {\overline{V}} \PIO({\overline{V}}^\top)) \\
& = & X {\overline{V}} \PIO({\overline{V}}^\top) {\overline{V}} {\overline{V}}^\top \\
& = & X{\overline{V}} (\PIO({\overline{V}}^\top) {\overline{V}}) {\overline{V}}^\top \\
& = & X {\overline{V}} G {\overline{V}}^\top,
\end{eqnarray*}
which leads to $\|\PP_{\overline{V}}\PIO\PP_{\overline{V}}(X)\| \leq
\psi$. Since $\psi<1$, $[I+\sum_{i=1}^\infty
(\PP_{\overline{V}}\PIO\PP_{\overline{V}})^i] (X) $ is well defined,
and has a spectral norm not larger than $1/(1-\psi)$.

Note that we have
\[
\PP_{\overline{V}}\PP_{\mathcal{I}_0^c}\PP_{\overline{V}}=\PP_{\overline{V}}(I-\PP_{\overline{V}}\PIO\PP_{\overline{V}}),
\] 
thus for any $X\in \PP_{\overline{V}}$ the following holds
\begin{eqnarray*}
\PP_{\overline{V}}\PP_{\mathcal{I}_0^c}\PP_{\overline{V}}[I+\sum_{i=1}^\infty (\PP_{\overline{V}}\PIO\PP_{\overline{V}})^i] (X)
 & = & \PP_{\overline{V}}(I-\PP_{\overline{V}}\PIO\PP_{\overline{V}})
[I+\sum_{i=1}^\infty (\PP_{\overline{V}}\PIO\PP_{\overline{V}})^i]
(X)\\
& = & \PP_{\overline{V}}(X)=X,
\end{eqnarray*} 
which establishes the lemma.
\end{proof}

\vspace{0.1in}

Now we define the matrices $\Delta_1$ and $\Delta_2$ used to construct the dual certificate. As the proof reveals, they are designed precisely as ``corrections'' to guarantee that the dual certificate satisfies the required constraints of Theorem \ref{thm:dualconditions}.

Define $\Delta_1$ and $\Delta_2$ as follows:
\begin{eqnarray}
\Delta_1 &\triangleq& \lambda \PP_{U_0}(H) =U_0\PIO({\overline{V}}^\top); \label{equ.delta1} \\
\Delta_2 &\triangleq& \PP_{U_0^\bot}\PP_{\mathcal{I}_0^c}\PP_{\overline{V}}[I+\sum_{i=1}^\infty
(\PP_{\overline{V}}\PIO\PP_{\overline{V}})^i]\PP_{\overline{V}}(\lambda
{\hat{H}}) \nonumber \\
&=&\PP_{\mathcal{I}_0^c}\PP_{\overline{V}}[I+\sum_{i=1}^\infty
(\PP_{\overline{V}}\PIO\PP_{\overline{V}})^i]\PP_{\overline{V}}\PP_{U_0^\bot}(\lambda
{\hat{H}}). \label{equ.delta2}
\end{eqnarray}
The equality holds since $\PP_{\overline{V}}, \PIO, \PP_{\mathcal{I}_0^c}$
are all given by right matrix multiplication, while $\PP_{U_0^\bot}$ is given by left matrix multiplication.

\begin{theorem}\label{thm.non-orth}
Assume $\psi<1$. Let
$$
Q \triangleq U_0{\overline{V}}^\top+ \lambda {\hat{H}}-\Delta_1-\Delta_2.
$$
If
$$
\frac{\gamma}{1-\gamma} \leq \frac{(1-\psi)^2}{(3-\psi)^2\mu r},
$$
and
$$
\frac{(1-\psi)\sqrt{\frac{\mu
r}{1-\gamma}}}{\sqrt{n}(1-\psi-\sqrt{\frac{\gamma}{1-\gamma}\mu
r})}\leq \lambda \leq \frac{1-\psi}{(2-\psi)\sqrt{n\gamma}},
$$
then $Q$ satisfies Condition~(\ref{equ.dualcertificate}) (i.e., it is the dual certificate). If all inequalities hold strictly, then $Q$ strictly satisfies~(\ref{equ.dualcertificate}).
\end{theorem}

\begin{proof} Note that $\psi < 1$  implies $\PP_{\overline{V}}\cap \PIO=\{0\}$. Hence it suffices to show that $Q$ simultaneously satisfies
\begin{eqnarray*}
(1) &&\PP_{\hat{U}}(Q)=\hat{U}{\hat{V}}^\top;\\
(2) &&\PP_{\hat{V}}(Q)=\hat{U}{\hat{V}}^\top;\\
(3) &&\PIO(Q)=\lambda {\hat{H}};\\
(4) && \|\PP_{\hat{T}^\bot}(Q)\| \leq 1;\\
(5) && \|\PP_{\mathcal{I}_0^c}(Q)\|_{\infty, 2} \leq \lambda.
\end{eqnarray*}
We prove that each of these five conditions holds, in Steps 1-5. Then in Step 6, we show that the condition on $\lambda$ is not vacuous, i.e., the lower bound is strictly less than then upper bound (and in fact, we then show that $\lambda = \frac{3}{7 \sqrt{\gamma n}}$ is in the specified range).

{\bf Step 1:} We have
\begin{eqnarray*}
\PP_{\hat{U}}(Q) & = & \PP_{U_0}(Q) \\
&=& \PP_{U_0}(U_0 {\overline{V}}^\top
+\lambda {\hat{H}}-\Delta_1-\Delta_2)\\
& = & U_0{\overline{V}}^\top + \lambda\PP_{U_0}({\hat{H}}) -\PP_{U_0}(\Delta_1)-\PP_{U_0}(\Delta_2)\\
& = & U_0{\overline{V}}^\top \\
&=& \hat{U}\hat{V}^\top.
\end{eqnarray*}

{\bf Step 2:} We have
\begin{eqnarray*}
\PP_{\hat{V}}(Q) & = & \PP_{\overline{V}}(Q) ~ = ~ \PP_{\overline{V}}(U_0 {\overline{V}}^\top +\lambda {\hat{H}}-\Delta_1-\Delta_2)\\
& = & U_0{\overline{V}}^\top+\PP_{\overline{V}}(\lambda {\hat{H}})-\PP_{\overline{V}}(\lambda
\PP_{U_0}({\hat{H}}))-\PP_{\overline{V}}(\PP_{\overline{V}}[I+\sum_{i=1}^\infty
(\PP_{\overline{V}}\PIO\PP_{\overline{V}})^i]\PP_{\overline{V}}\PP_{U_0^\bot}(\lambda {\hat{H}}))\\
& = &U_0{\overline{V}}^\top+\PP_{\overline{V}}(\PP_{U_0^\bot}(\lambda
{\hat{H}}))-\PP_{\overline{V}}\PP_{\mathcal{I}_0^c}\PP_{\overline{V}}[I+\sum_{i=1}^\infty
(\PP_{\overline{V}}\PIO\PP_{\overline{V}})^i]\PP_{\overline{V}}\PP_{U_0^\bot}(\lambda {\hat{H}})\\
&\stackrel{(a)}{=} & U_0{\overline{V}}^\top+\PP_{\overline{V}}(\PP_{U_0^\bot}(\lambda
{\hat{H}}))-\PP_{\overline{V}}(\PP_{U_0^\bot}(\lambda
{\hat{H}})) \\
&=& U_0{\overline{V}}^\top \\ 
&=& \hat{U}\hat{V}^\top.
\end{eqnarray*}
Here, (a) holds since on $\PP_{\overline{V}}$, $[I+\sum_{i=1}^\infty
(\PP_{\overline{V}}\PIO\PP_{\overline{V}})^i]$ is the inverse
operation of
$\PP_{\overline{V}}\PP_{\mathcal{I}_0^c}\PP_{\overline{V}}$.

{\bf Step 3:} We have
\begin{eqnarray*}
\PIO(Q) 
& = &\PIO(U_0{\overline{V}}^\top+\lambda {\hat{H}}-\Delta_1-\Delta_2)\\
& = & U_0\PIO({\overline{V}}^\top)+\lambda
{\hat{H}}-\PIO(U_0\PIO({\overline{V}}^\top))-\PIO\PP_{\mathcal{I}_0^c}\PP_{\overline{V}}[I+\sum_{i=1}^\infty
(\PP_{\overline{V}}\PIO\PP_{\overline{V}})^i]\PP_{\overline{V}}\PP_{U_0^\bot}(\lambda
{\hat{H}})\\
& = &\lambda {\hat{H}}.
\end{eqnarray*}

{\bf Step 4:} We need a lemma first.
\begin{lemma}Given
 $X\in \mathbb{R}^{p\times n}$ such that
$\|X\|=1$, we have $\|\PP_{\mathcal{I}_0^c}\PP_{\overline{V}}(X)\|\leq 1$.\end{lemma}
\begin{proof} By definition,
$$
\PP_{\mathcal{I}_0^c}\PP_{\overline{V}}(X)=X
{\overline{V}}\PP_{\mathcal{I}_0^c}({\overline{V}}^\top).
$$
For any
$\mathbf{z}\in \mathbb{R}^n$ such that $\|\mathbf{z}\|_2=1$, we have
$$
\|X {\overline{V}} \PP_{\mathcal{I}_0^c}({\overline{V}}^\top)
\mathbf{z}\|_2 =\|X {\overline{V}} {\overline{V}}^\top
\PP_{\mathcal{I}_0^c}(\mathbf{z})\|_2 \leq \|X\| \|{\overline{V}}
{\overline{V}}^\top\| \|\PP_{\mathcal{I}_0^c}(\mathbf{z})\|_2 \leq
1,
$$
where we use $\PP_{\mathcal{I}_0^c}(\mathbf{z})$ to represent the vector whose coordinates $i \in \mathcal{I}_0$ are set to zero. The last inequality follows from the fact that $\|X\|=1$. Note that this holds for any $\mathbf{z}$,
hence by the definition of spectral norm (as the $\ell_2$ operator
norm), the lemma follows.
\end{proof}

Now we continue with Step 4. We have
\begin{eqnarray*}
\PP_{\hat{T}^\bot}(Q) & = &\PP_{\hat{T}^\bot}(U_0{\overline{V}}^\top+\lambda
{\hat{H}}-\Delta_1-\Delta_2)\\
& = &\PP_{{\overline{V}}^\bot}\PP_{U_0^\bot}(\lambda {\hat{H}})-
\PP_{{\overline{V}}^\bot}\PP_{U_0^\bot}(\PP_{\mathcal{I}_0^c}\PP_{\overline{V}}[I+\sum_{i=1}^\infty
(\PP_{\overline{V}}\PIO\PP_{\overline{V}})^i]\PP_{\overline{V}}\PP_{U_0^\bot}(\lambda {\hat{H}}))\\
& = &\PP_{{\overline{V}}^\bot}\PP_{U_0^\bot}(\lambda {\hat{H}})-
\PP_{U_0^\bot}\PP_{{\overline{V}}^\bot}\PP_{\mathcal{I}_0^c}\PP_{\overline{V}}[I+\sum_{i=1}^\infty
(\PP_{\overline{V}}\PIO\PP_{\overline{V}})^i]\PP_{\overline{V}}(\lambda
{\hat{H}}).
\end{eqnarray*}
Let $v=\|\lambda {\hat{H}}\|$. Recall that we have shown $v \leq \lambda \sqrt{|\mathcal{I}_0|}$. Thus we have
$\|\PP_{{\overline{V}}^\bot}\PP_{U_0^\bot}(\lambda {\hat{H}})\|\leq
v$. Furthermore, we have the following:
\begin{eqnarray*}
\|\PP_{\overline{V}}(\lambda {\hat{H}})\| \leq v & \Longrightarrow & \|[I+\sum_{i=1}^\infty 
(\PP_{\overline{V}}\PIO\PP_{\overline{V}})^i]\PP_{\overline{V}}(\lambda
{\hat{H}})\| \leq v/(1-\psi)  \\
& \Longrightarrow &
\|\PP_{\mathcal{I}_0^c}\PP_{\overline{V}}[I+\sum_{i=1}^\infty
(\PP_{\overline{V}}\PIO\PP_{\overline{V}})^i]\PP_{\overline{V}}(\lambda
{\hat{H}})\| \leq v/(1-\psi)  \\
& \Longrightarrow &\|\PP_{U_0^\bot}\PP_{{\overline{V}}^\bot}\PP_{\mathcal{I}_0^c}\PP_{\overline{V}}[I+\sum_{i=1}^\infty
(\PP_{\overline{V}}\PIO\PP_{\overline{V}})^i]\PP_{\overline{V}}(\lambda
{\hat{H}})\| \leq v/(1-\psi).
\end{eqnarray*}
Thus we have that
$$
\|\PP_{\hat{T}^\bot}(Q)\|\leq \frac{2-\psi}{1-\psi}\lambda
\sqrt{|\mathcal{I}_0|}.
$$
From the assumptions of the theorem, we have
$$
\lambda \leq \frac{1-\psi}{(2-\psi)\sqrt{n\gamma}},
$$
and hence
$$
\|\PP_{\hat{T}^\bot}(Q)\| \leq 1.
$$
The inequality will be strict if
$$
\lambda < \frac{1-\psi}{(2-\psi)\sqrt{n\gamma}}.
$$

{\bf Step 5:} We first need a lemma that shows that the incoherence parameter for the matrix $\overline{V}$ is no larger than the incoherence parameter of the original matrix $V_0$.

\begin{lemma} 
\label{lem.newincoherence}
Define the incoherence of $\overline{V}$ as follows:
$$
\overline{\mu}={\max_{i\in \mathcal{I}_0^c}}\frac{|\mathcal{I}_0^c|}{r} \| \PP_{\mathcal{I}_0^c}({\overline{V}}^\top)\mathbf{e}_i\|^2.
$$
Then $\overline{\mu} \leq \mu$.
\end{lemma}

\begin{proof}
Recall that $L_0=U_0 \Sigma_0 V_0^{\top}$, and
$$
\mu={\max_{i\in \mathcal{I}_0^c}}\frac{|\mathcal{I}_0^c|}{r} \|
\PP_{\mathcal{I}_0^c}(V_0^\top)\mathbf{e}_i\|^2.
$$
Thus it suffices to show that for fixed $i \in \mathcal{I}_0$, the
following holds:
$$
\| \PP_{\mathcal{I}_0^c}({\overline{V}}^\top)\mathbf{e}_i\| \leq \|
\PP_{\mathcal{I}_0^c}(V_0^{\top})\mathbf{e}_i\|.
$$
 Note that
$\PP_{\mathcal{I}_0^c}({\overline{V}}^\top)$ and
$\PP_{\mathcal{I}_0^c}(V_0^{\top})$ span the same row space. Thus,
due to the fact that $\PP_{\mathcal{I}_0^c}(V_0^{\top})$ is
orthonormal, we conclude that $\PP_{\mathcal{I}_0^c}({\overline{V}}^\top)$ is
row-wise full rank. Since $0\preceq \PP_{\mathcal{I}_0^c}({\overline{V}}^\top)\PP_{\mathcal{I}_0^c}({\overline{V}}^\top)^\top =I-G$,
and $G \succeq 0$, there exists a symmetric, invertible matrix $Y\in \mathbb{R}^{r\times r}$, such that
$$
\|Y\|\leq 1;\quad\mbox{and}\quad
Y^2=\PP_{\mathcal{I}_0^c}({\overline{V}}^\top)\PP_{\mathcal{I}_0^c}({\overline{V}}^\top)^\top.
$$
This in turn implies that $Y^{-1}\PP_{\mathcal{I}_0^c}({\overline{V}}^\top)$ is orthonormal
and spans the same row space as
$\PP_{\mathcal{I}_0^c}({\overline{V}}^\top)$, and hence spans the
same row space as $\PP_{\mathcal{I}_0^c}(V_0^{\top})$. Note that
$\PP_{\mathcal{I}_0^c}(V_0^{\top})$ is also orthonormal, which
implies there exists an orthonormal matrix $Z \in
\mathbb{R}^{r\times r}$, such that
$$
Z
Y^{-1}\PP_{\mathcal{I}_0^c}({\overline{V}}^\top)=\PP_{\mathcal{I}_0^c}(V_0^{\top}).
$$
We have
$$
\|\PP_{\mathcal{I}_0^c}({\overline{V}}^\top) \mathbf{e}_i\|_2 =\|Y
Z^\top \PP_{\mathcal{I}_0^c}(V_0^{\top}) \mathbf{e}_i\|_2\leq
\|Y\|\|Z^\top\|\|\PP_{\mathcal{I}_0^c}(V_0^{\top})
\mathbf{e}_i\|_2\leq \|\PP_{\mathcal{I}_0^c}(V_0^{\top})
\mathbf{e}_i\|_2.
$$
This concludes the proof of the lemma.
\end{proof}
\vspace{0.1cm}
Now, recall from the proof of Lemma~\ref{lem.critical} that
\[\PP_{\overline{V}}\PIO\PP_{\overline{V}}(X)=X {\overline{V}} G {\overline{V}}^\top.\]
Hence, noting that $(\PP_{\overline{V}}\PIO\PP_{\overline{V}})^i =
(\PP_{\overline{V}}\PIO\PP_{\overline{V}})(\PP_{\overline{V}}\PIO\PP_{\overline{V}})^{i-1}$
and ${\overline{V}}^\top {\overline{V}}=I$, by induction we have
$$
(\PP_{\overline{V}}\PIO\PP_{\overline{V}})^i(X)=X {\overline{V}} G^i {\overline{V}}^\top.
$$
We use this to expand $\Delta_2$:
\begin{eqnarray*}
\Delta_2 & = & \PP_{U^\bot}\PP_{\mathcal{I}_0^c}\PP_{\overline{V}}[I+\sum_{i=1}^\infty
(\PP_{\overline{V}}\PIO\PP_{\overline{V}})^i]\PP_{\overline{V}}(\lambda {\hat{H}})\\
& = & (I-U_0U_0^\top)(\lambda {\hat{H}}) {\overline{V}}
{\overline{V}}^\top[1+ \sum_{i=1}^{\infty} {\overline{V}} G^i
{\overline{V}}^\top]{\overline{V}} \PP_{\mathcal{I}_0^c}({\overline{V}}^\top).
\end{eqnarray*}
Thus, we have
\begin{eqnarray*}
\|\Delta_2 \mathbf{e}_i\|_2 
& \leq & \|(I-U_0U_0^\top)\| \|(\lambda H)\| \|{\overline{V}}
{\overline{V}}^\top\| \|1+ \sum_{i=1}^{\infty} {\overline{V}} G^i
{\overline{V}}^\top\| \|{\overline{V}}\| \|\PP_{\mathcal{I}_0^c}({\overline{V}}^\top)
\mathbf{e}_i\|_2\\
& \leq & \|\lambda H\| \frac{1}{1-\psi} \sqrt{\frac{\overline{\mu}
r}{n-|\mathcal{I}_0|}}\\
& \leq & \frac{\lambda \sqrt{|\mathcal{I}_0|}\sqrt{\frac{\mu
r}{n-|\mathcal{I}_0|}}}{1-\psi},
\end{eqnarray*}
where we have used Lemma \ref{lem.newincoherence} in the last inequality. This now implies
$$
\|\Delta_2\|_{\infty, 2} \leq \frac{\lambda
\sqrt{|\mathcal{I}_0|}\sqrt{\frac{\mu
r}{n-|\mathcal{I}_0|}}}{1-\psi}.
$$
Notice that
\begin{equation*}\begin{split}
\|\PP_{\mathcal{I}_0^c}(Q)\|_{\infty,2}=&\|\PP_{\mathcal{I}_0^c}(U_0{\overline{V}}^\top +\lambda {\hat{H}}- \Delta_1-\Delta_2)\|_{\infty,2}\\
=&\|U_0\PP_{\mathcal{I}_0^c}({\overline{V}}^\top)-\Delta_2\|_{\infty,2}\\
\leq & \|U_0\PP_{\mathcal{I}_0^c}({\overline{V}}^\top)\|_{\infty,2}+\|\Delta_2\|_{\infty,2}\\
\leq &  \sqrt{\frac{\mu  r}{n-|\mathcal{I}_0|}}+\frac{\lambda
\sqrt{|\mathcal{I}_0|}\sqrt{\frac{\mu
r}{n-|\mathcal{I}_0|}}}{1-\psi}.
\end{split}\end{equation*}
Therefore, showing that $\|\PP_{\mathcal{I}_0^c}(Q)\|_{\infty,2} \leq \lambda$ is equivalent to showing

\begin{equation*}\begin{split}&\sqrt{\frac{\mu
r}{n-|\mathcal{I}_0|}}+\frac{\lambda
\sqrt{|\mathcal{I}_0|}\sqrt{\frac{\mu r}{n-|\mathcal{I}_0|}}}{1-\psi}\leq \lambda\\\Longleftrightarrow
\qquad & \lambda \left(1-\frac{\sqrt{\frac{\gamma}{1-\gamma}\mu
r}}{1-\psi}\right)\geq
\sqrt{\frac{\mu r}{n(1-\gamma)}}\\
\Longleftrightarrow \qquad & \lambda \geq
\frac{(1-\psi)\sqrt{\frac{\mu
r}{1-\gamma}}}{\sqrt{n}(1-\psi-\sqrt{\frac{\gamma}{1-\gamma}\mu
r})},\end{split}\end{equation*} 
as long as $1-\psi-\sqrt{\frac{\gamma}{1-\gamma}\mu  r}>0$ (which is proved in Step 6).

{\bf Step 6:} We have shown that each of the 5 conditions holds. Finally, we show that the theorem's conditions on $\lambda$ can be satisfied. But this amounts to a condition on $\gamma$. Indeed, we have:
\begin{equation*}\begin{split}&\frac{(1-\psi)\sqrt{\frac{\mu
r}{1-\gamma}}}{\sqrt{n}(1-\psi-\sqrt{\frac{\gamma}{1-\gamma}\mu
r})}
\leq \frac{1-\psi}{(2-\psi)\sqrt{n\gamma}}\\
\Longleftrightarrow \qquad &
(2-\psi)\sqrt{\frac{\gamma}{1-\gamma}\mu r}\leq
1-\psi-\sqrt{\frac{\gamma}{1-\gamma}\mu r}\\
\Longleftrightarrow \qquad & \frac{\gamma}{1-\gamma} \leq
\frac{(1-\psi)^2}{(3-\psi)^2\mu r},
\end{split}
\end{equation*}
which can certainly be satisfied, since the right hand side does not depend on $\gamma$. Moreover,  observe that under this condition,
$1-\psi-\sqrt{\frac{\gamma}{1-\gamma}\mu r}>0$ holds. Note that if the last inequality holds strictly, then so does the first.
\end{proof}

We have thus shown that as long as $\psi < 1$, then for $\lambda$ within the given bounds, we can construct a dual certificate. From here, the following corollary immediately establishes our main result, Theorem~\ref{thm.noiseless}.
\begin{corollary}Let $\gamma\leq \gamma^*$, then Outlier Pursuit, with
$\lambda=\frac{3}{7\sqrt{\gamma^*n}}$, strictly succeeds if
\[\frac{\gamma^*}{1-\gamma^*} \leq
\frac{9}{121\mu  r}.\]
\end{corollary}
\begin{proof}First note that
$\lambda=\frac{3}{7\sqrt{\gamma^*n}}$ and $\gamma\leq \gamma^*$
imply that
$$
\lambda \leq \frac{3}{7\sqrt{\gamma n}},
$$
which by Lemma~\ref{lemma.c} leads to
$$
\psi\leq \lambda^2 \gamma n < \frac{1}{4}.
$$
Thus, it suffices to check that $\gamma$ and $\lambda$ satisfies the
conditions of Theorem~\ref{thm.non-orth}, namely
$$
\frac{\gamma}{1-\gamma} < \frac{(1-\psi)^2}{(3-\psi)^2\mu r},
$$
 and
 $$
\frac{(1-\psi)\sqrt{\frac{\mu
r}{1-\gamma}}}{\sqrt{n}(1-\psi-\sqrt{\frac{\gamma}{1-\gamma}\mu
r})}< \lambda < \frac{1-\psi}{(2-\psi)\sqrt{n\gamma}}.
$$

Since   $\psi< 1/4$, we have
$$
\frac{\gamma}{1-\gamma}\leq\frac{\gamma^*}{1-\gamma^*}  \leq
\frac{9}{121\mu  r}=\frac{(1-1/4)^2}{(3-1/4)^2\mu  r }<
\frac{(1-\psi)^2}{(3-\psi)^2\mu  r },
$$
which proves the first condition.

Next, observe that $\frac{(1-\psi)\sqrt{\frac{\mu
r}{1-\gamma}}}{\sqrt{n}(1-\psi-\sqrt{\frac{\gamma}{1-\gamma}\mu
r})}$, as a function of $\psi, \gamma, (\mu  r)$ is strictly
increasing in $\psi$, $(\mu  r)$, and $\gamma$. Moreover, $\mu  r\leq
\frac{(1-\psi)^2(1-\gamma)}{(3-\psi)^2\gamma}$, and thus
$$
\frac{(1-\psi)\sqrt{\frac{\mu
r}{1-\gamma}}}{\sqrt{n}(1-\psi-\sqrt{\frac{\gamma}{1-\gamma}\mu r})}
<\frac{(1-\psi)\sqrt{\frac{(1-\psi)^2}{(3-\psi)^2\gamma}}}{\sqrt{n}(1-\psi-\frac{1-\psi}{3-\psi})}
=\frac{3\sqrt{1+\gamma/(1-\gamma)}}{7\sqrt{n}} \leq
\frac{3\sqrt{1+\gamma^*/(1-\gamma^*)}}{7\sqrt{n}}=\lambda.
$$

Similarly, $\frac{1-\psi}{(2-\psi)\sqrt{n\gamma}}$ is strictly
decreasing in $\psi$ and $\gamma$, which implies that
\[\frac{1-\psi}{(2-\psi)\sqrt{n\gamma}}> \frac{1-1/4}{(2-1/4)\sqrt{n \gamma^*}}=\lambda.\]
\end{proof}

\section{Proof of Theorem~\ref{thm.noise}: The Case of Noise}

In practice, the observed matrix may be a noisy copy of $M$. In this section, we
investigate this noisy case and show that the proposed method, with
minor modification, is robust to noise. Specifically,
  we observe $M'=M+N$ for some
unknown $N$, and we want to approximately recover $U_0$ and
$\mathcal{I}_0$. This leads to the following formulation that
replaces the equality constraint $M=L+C$ with a norm inequality.
\begin{equation}\label{equ.form.noisy}\begin{split}
\mbox{Minimize:}\quad & \|L\|_*+\lambda \|C\|_{1,2} \\
\mbox{Subject to:}\quad &  \|M'-L-C\|_F \leq \epsilon.
\end{split}\end{equation}
In fact, we show in this section that under the
essentially equivalent conditions as that of the noiseless case,
Noisy Outlier Pursuit succeeds. Here, we say that the algorithm ``succeeds'' if the
optimal solution of~(\ref{equ.form.noisy}) is ``close'' to a pair that has the correct column space and column support.
To this end, we first establish the next theorem~--~a
counterpart in the noisy case of
Theorem~\ref{thm:dualconditions}~--~that states that  Noisy Outlier
Pursuit succeeds if there exists a dual certificate (with slightly
stronger requirements than the noiseless case) for decomposing the
{\em noiseless matrix $M$}. Then, applying our results  on constructing
the dual certificate from the previous section, we have that Noisy
Outlier Pursuit succeeds under the essentially equivalent conditions
as that of the noiseless case.
\begin{theorem}\label{thm.noisein}Let $L', C'$ be an optimal solution  of~(\ref{equ.form.noisy}).
Suppose $\|N\|_F\leq \epsilon$, $\lambda<1$, and $\psi<1/4$. Let
$M=\hat{L}+\hat{C}$ where $\PP_U(\hat{L})=\hat{L}$ and
$\PIO(\hat{C})=\hat{C}$. If there exists a $Q$ such that
\begin{equation}\label{equ.dualcertificate.nosiy}\begin{split}
 &\mathcal{P}_{T(\hat{L})}(Q) =\mathfrak{N}(\hat{L}); \quad \|\mathcal{P}_{T(\hat{L})^\bot}(Q)\| \leq
1/2;\quad \mathcal{P}_{\mathcal{I}_0}(Q)/\lambda \in
\mathfrak{G}(\hat{C}); \quad
\|\mathcal{P}_{\mathcal{I}_0^c}(Q)\|_{\infty,2}\leq \lambda/2,
\end{split}\end{equation}
then there exists a pair $(\tilde{L}, \tilde{C})$ such that
$M=\tilde{L}+\tilde{C}$, $\tilde{L}\in \PP_{U_0}$, $\tilde{C}\in \PIO$
and
$$
\| L'-\tilde{L}\|_F\leq 10\sqrt{n} \epsilon;\quad \|
C'-\tilde{C}\|_F\leq 9\sqrt{n} \epsilon.
$$
\end{theorem}

\begin{proof} Let $\overline{V}$ be as defined before.
We establish the following lemma first.
\begin{lemma} Recall  that $\psi=\|G\|$ where $G=\PIO({\overline{V}}^\top) \PIO({\overline{V}}^\top)^\top$. We have
$$
\|\PIO\PP_{\overline{V}}\PIO(X)\|_F\leq \psi \|X\|_F.
$$
\end{lemma}
\begin{proof}Let $T\in \mathbb{R}^{n\times n}$ be such that
\[T_{ij}=\left\{\begin{array}{cc} 1 &\mbox{if}\,\, i=j,\, i\in
\mathcal{I};\\ 0 &\mbox{otherwise.}\end{array}\right.\]
 We then expand $\PIO \PP_{\overline{V}}\PIO(X)$, which equals
 $$
 X T {\overline{V}} {\overline{V}}^\top T=X T {\overline{V}} {\overline{V}}^\top T^\top
 =X (T{\overline{V}}) (T{\overline{V}})^\top=X \PIO({\overline{V}}^\top)^\top \PIO({\overline{V}}^\top).
 $$
The last equality follows from
$(T{\overline{V}})^{\top}=\PIO({\overline{V}}^\top)$. Since
$\psi=\|G\|$ where $G=\PIO({\overline{V}}^\top)
\PIO({\overline{V}}^\top)^\top$, we have
$$
\|\PIO({\overline{V}}^\top)^\top
\PIO({\overline{V}}^\top)\|=\|\PIO({\overline{V}}^\top)
\PIO({\overline{V}}^\top)^\top\|=\psi.
$$
Now consider the $i^{th}$ row of $X$, denoted as $\mathbf{x}^i$.
Since $\|\PIO({\overline{V}}^\top)^\top
\PIO({\overline{V}}^\top)\|=\psi$, we have
\[\|\mathbf{x}^i \PIO({\overline{V}}^\top)^\top \PIO({\overline{V}}^\top)\|_2^2 \leq  \psi^2
\|\mathbf{x}^i\|_2^2.
\]The lemma holds from the following inequality.
\[\|\PIO \PP_{\overline{V}}\PIO(X)\|_F^2=  \|X \PIO({\overline{V}}^\top)^\top \PIO({\overline{V}}^\top)\|_F^2
 =\sum_{i} \|\mathbf{x}_i \PIO({\overline{V}}^\top)^\top \PIO({\overline{V}}^\top)\|_2^2 \leq \psi^2 \sum \|\mathbf{x}^i\|_2^2=\psi^2 \|X\|_F^2.\]
\end{proof}

Let $N_L=L'-\hat{L}$, and $N_C=C'-\hat{C}$. Thus $N=N_C+N_L$, and
recall that $\|N\|_F\leq \epsilon$. Further, define
$N_L^+=N_L-\PIO\PP_{U_0}(N_L)$, $N_C^+=N_C-\PIO\PP_{U_0}(N_C)$, and
$N^+=N-\PIO\PP_{U_0}(N)$. Observe that for any $A$,
$\|(I-\PIO\PP_{U_0})(A)\|_F\leq \|A\|_F$.

Choosing the same $W$ and $F$ as in the proof of Theorem~\ref{thm:dualconditions}, we have
\begin{equation*}
\begin{split}
&\|\hat{L}\|_*+\lambda \|\hat{C}\|_{1,2} \geq  \|L'\|_*+\lambda \|C'\|_{1,2}\\
\geq & \|\hat{L}\|_*+\lambda \|\hat{C}\|_{1,2} + \langle
\PP_{T(\hat{L})}(Q)+W,
N_L\rangle +\lambda\langle \PIO(Q)/\lambda +F, N_C\rangle\\
= & \|\hat{L}\|_*+\lambda \|\hat{C}\|_{1,2} +
\|\PP_{T(\hat{L})^\bot}(N_L)\|_{*}+\lambda
\|\PP_{\mathcal{I}_0^c}(N_C)\|_{1,2}+\langle
\PP_{T(\hat{L})}(Q), N_L\rangle + \langle \PIO(Q), N_C \rangle\\
= & \|\hat{L}\|_*+\lambda \|\hat{C}\|_{1,2} +
\|\PP_{{T(\hat{L})}^\bot}(N_L)\|_{*}+\lambda
\|\PP_{\mathcal{I}_0^c}(N_C)\|_{1,2}-\langle
\PP_{{T(\hat{L})}^\bot}(Q), N_L\rangle - \langle
\PP_{\mathcal{I}_0^c}(Q), N_C \rangle+\langle
Q, N_L+N_C \rangle\\
\geq & \|\hat{L}\|_*+\lambda \|\hat{C}\|_{1,2}
+(1-\|\PP_{{T(\hat{L})}^\bot}(Q)\|)
\|\PP_{{T(\hat{L})}^\bot}(N_L)\|_{*}+(\lambda-
\|\PP_{\mathcal{I}_0^c}(Q)\|_{\infty, 2}) \|\PP_{\mathcal{I}_0^c}(N_C)\|_{1,2}+\langle Q, N\rangle\\
\geq & \|\hat{L}\|_*+\lambda \|\hat{C}\|_{1,2} +(1/2)
\|\PP_{{T(\hat{L})}^\bot}(N_L)\|_{*}+(\lambda/2)
\|\PP_{\mathcal{I}_0^c}(N_C)\|_{1,2}- \epsilon\|Q\|_F.
\end{split}
\end{equation*}
Note that $\|Q\|_{\infty, 2} \leq \lambda$, hence $\|Q\|_F \leq
\sqrt{n} \lambda$. Thus we have
\begin{equation}\begin{split}
&\|\PP_{{T(\hat{L})}^\bot}(N_L)\|_F \leq
\|\PP_{{T(\hat{L})}^\bot}(N_L)\|_* \leq 2\lambda
\sqrt{n} \epsilon;\\
& \|\PP_{\mathcal{I}_0^c}(N_C)\|_F \leq
\|\PP_{\mathcal{I}_0^c}(N_C)\|_{1,2} \leq 2  \sqrt{n} \epsilon.
\end{split}\end{equation}
Furthermore,
\begin{equation}\label{equ.noisycaseproof}
\begin{split}
\PIO(N_C^+)=&\PIO(N_C)-\PIO\PP_{U_0}\PIO(N_C)\\
=&\PIO(N)-\PIO\PP_{{T(\hat{L})}^\bot}(N_L)-\PIO\PP_{{T(\hat{L})}}(N_L)-\PIO\PP_{U_0}\PIO(N_C)\\
=&\PIO(N)-\PIO\PP_{{T(\hat{L})}^\bot}(N_L)-\PIO\PP_{T}(\Delta)+\PIO\PP_{T(\hat{L})}(N_C)
-\PIO\PP_{U_0}\PIO(N_C)\\
=&\PIO(N)-\PIO\PP_{{T(\hat{L})}^\bot}(N_L)-\PIO\PP_{T(\hat{L})}(\Delta)+
\PIO\PP_{T(\hat{L})}\PP_{\mathcal{I}_0^c}(N_C)\\&\quad+\PIO\PP_{T(\hat{L})}\PIO(N_C)
-\PIO\PP_{U_0}\PIO(N_C)\\
\stackrel{(a)}{=}&\PIO(N)-\PIO\PP_{{T(\hat{L})}^\bot}(N_L)-\PIO\PP_{T(\hat{L})}(\Delta)
+\PIO\PP_{T(\hat{L})}\PP_{\mathcal{I}_0^c}(N_C)+\PIO\PP_{T(\hat{L})}\PIO(N_C^+)\\
\stackrel{(b)}{=}&\PIO(N)-\PIO\PP_{{T(\hat{L})}^\bot}(N_L)-\PIO\PP_{T(\hat{L})}(\Delta)
+\PIO\PP_{T(\hat{L})}\PP_{\mathcal{I}_0^c}(N_C)+\PIO\PP_{\overline{V}}\PIO(N_C^+).
\end{split}
\end{equation}
Here (a) holds due to the following
\begin{equation*}
\PIO\PP_{T(\hat{L})}\PIO(N_C^+)=\PIO\PP_{T(\hat{L})}\PIO(N_C)-\PIO\PP_{T(\hat{L})}\PIO(\PIO\PP_{U_0}(N_C))
=\PIO\PP_{T(\hat{L})}\PIO(N_C)
 -\PIO\PP_{U_0}\PIO(N_C),\end{equation*}
 and (b) holds since by definition, each column of $N_C^+$ is
 orthogonal to $U_0$, hence $\PP_{U_0}\PIO(N_C^+)=0$. Thus,
 Equation~(\ref{equ.noisycaseproof}) leads to
 \begin{equation*}\begin{split}&\|\PIO(N_C^+)\|_F\\ \leq
 &\|\PIO(N)-\PIO\PP_{T(\hat{L})}(N)\|_F+\|\PIO\PP_{{T(\hat{L})}^\bot}(N_L)\|_F+
 \|\PIO\PP_{T(\hat{L})}\PP_{\mathcal{I}_0^c}(N_C)\|_F+\|\PIO\PP_{{\overline{V}}}\PIO(N_C^+)\|_F\\
 \leq&
 \|N\|_F+\|\PP_{{T(\hat{L})}^\bot}(N_L)\|_F+\|\PP_{\mathcal{I}_0^c}(N_C)\|_F+\psi
 \|\PIO(N_C^+)\|_F\\
 \leq & (1+2\lambda \sqrt{n} + 2\sqrt{n})\epsilon + \psi
 \|\PIO(N_C^+)\|_F.
 \end{split}\end{equation*}
 This implies that
 \[\|\PIO(N_C^+)\|_F \leq (1+2\lambda \sqrt{n} + 2\sqrt{n})\epsilon/(1-\psi).\]
 Now using the fact that $\lambda <1$, and $\psi<1/4$, we have
 \[\|N_C^+\|_F = \|\PP_{\mathcal{I}_0^c}(N_C)+\PIO(N_C^+)\|_F\leq \|\PP_{\mathcal{I}_0^c}(N_C)\|_F+\|\PIO(N_C^+)\|_F \leq 9\sqrt{n}\epsilon.\]
 Note that $N_C^+= (I-\PIO\PP_{U_0})(C'-\hat{C}) =
 C'-[\hat{C}+\PIO\PP_{U_0}(C'-\hat{C})]$. Letting
 $\tilde{C}=\hat{C}+\PIO\PP_{U_0}(C'-\hat{C})$, we have $\tilde{C}\in
 \PIO$ and $\|C-\tilde{C}\|_F \leq 9\sqrt{n}\epsilon.$
Letting $\tilde{L}=\hat{L}-\PIO\PP_{U_0}(C'-\hat{C})$, we have that
$\tilde{L}, \tilde{C}$ is a successful decomposition, and
\[\|L'-\tilde{L}\|_F \leq \|N\|_F +\|C'-\tilde{C}\|_F\leq
10\sqrt{n}\epsilon.\]
\end{proof}

{\bf Remark:} From the proof of Theorem~\ref{thm.non-orth}, we have
that Condition~(\ref{equ.dualcertificate.nosiy}) holds when
\[\frac{\gamma}{1-\gamma} \leq
\frac{(1-\psi)^2}{(9-4\psi)^2\mu_0r}\]and
\[\frac{2(1-\psi)\sqrt{\frac{\mu_0
r}{1-\gamma}}}{\sqrt{n}(1-\psi-\sqrt{\frac{\gamma}{1-\gamma}\mu_0
r})}\leq \lambda \leq \frac{1-\psi}{2(2-\psi)\sqrt{n\gamma}}.\] For
example, one can take
\[\lambda=\frac{\sqrt{9+1024\mu_0 r}}{14\sqrt{n}},\]
and all conditions of Theorem~\ref{thm.noisein} hold when
\[\frac{\gamma}{1-\gamma} \leq \frac{9}{1024\mu_0 r}.\] This
establishes Theorem~\ref{thm.noise}.

\section{Implementation issues and numerical experiments}
While minimizing the nuclear norm is known to be a semi-definite
program, and can be solved using a general purpose SDP solver such as
SDPT3 or SeDuMi, such a method does not scale well to large
data-sets. In fact, the computational time becomes prohibitive even
for modest problem sizes as small as hundreds of variables. Recently, a family of
optimization algorithms known as {\em proximal gradient algorithms}
have been proposed to solve optimization problems of the form
\[\mbox{minimize:} \,\, g(\mathbf{x}),\quad \mbox{subject to:}\,\, \mathcal{A}(\mathbf{x})=\mathbf{b},\]
of which Outlier Pursuit is a special case. It is known that such
algorithms converge with a rate of $O(k^{-2})$ where $k$ is the number of variables, and significantly
outperform interior point methods for solving SDPs in practice. Following this paradigm, we solve Outlier Pursuit with the following
algorithm. The validity of the algorithm  follows easily
from~\cite{Tsang08,Nesterov83}. See also~\cite{CaiCandesShen08}.

\begin{algorithm}
\begin{list}\usecount{}
   \item[{\bf Input:}] $M\in \mathbb{R}^{m\times n}$, $\lambda$,
   $\delta:=10^{-5}$, $\eta:=0.9$, $\mu_0:=0.99 \|M\|_F$.
   \item[]
   \begin{enumerate}
      \item $L_{-1}, L_0:=0^{m\times n}$; $C_{-1}, C_0:=0^{m\times
      n}$, $t_{-1}, t_0:=1$; $\bar{\mu}=\delta \mu$;
      \item {\bf while} not converged {\bf do}
      \item $Y^L_k:= L_k+\frac{t_{k-1}-1}{t_k}(L_k-L_{k-1})$, $Y^C_k:=
      C_k+\frac{t_{k-1}-1}{t_k}(C_k-C_{k-1})$;
      \item $G^L_k:=Y^L_k-\frac{1}{2}\big(Y^L_k+Y^C_k-M\big)$; $G^C_k:=Y^C_k-\frac{1}{2}\big(Y^L_k+Y^C_k-M\big)$;
      \item $(U, S, V):=\mathrm{svd}(G^L_k)$; $L_{k+1}:= U \mathfrak{L}_{\frac{\mu_k}{2}}(S)
      V$;
      \item $C_{k+1}:=\mathfrak{C}_{\frac{\lambda
      \mu_k}{2}}(G^C_k)$;
      \item $t_{k+1}:=\frac{1+\sqrt{4t_k^2+1}}{2}$; $\mu_{k+1}:=\max (\eta\mu_k. \bar{\mu})$;
      $k++$;
      \item {\bf end while}
   \end{enumerate}
   \item[{\bf Output:}] $L:=L_k$, $C=C_k$.
\end{list}
\end{algorithm}
Here, $\mathfrak{L}_{\epsilon}(S)$ is the diagonal soft-thresholding
operator: if $|S_{ii}|\leq \epsilon$, then it is set to zero,
otherwise, we set $S_{ii}:=S_{ii}-\epsilon \cdot\mathrm{sgn}(S_{ii})$.
Similarly, $\mathfrak{C}_{\epsilon}(C)$ is the column-wise
thresholding operator: set $C_i$ to zero if $\|C_i\|_2\leq
\epsilon$, otherwise set $C_i:= C_i - \epsilon C_i/{\|C_i\|_2}$.

We explore the performance of Outlier Pursuit on some synthetic and real-world data, and find that its performance is quite promising.\footnote{We have learned that \cite{Tropp2010} has also performed some numerical experiments minimizing $\|\cdot\|_{\ast} + \lambda \|\cdot\|_{1,2}$, and found promising results.}
Our first experiment investigates the phase-transition property of
Outlier Pursuit, using randomly generated synthetic
data. Fix $n=p=400$. For different $r$ and  number
of outliers $\gamma n$, we generated matrices $A\in
\mathbb{R}^{p\times r}$ and $B\in \mathbb{R}^{(n-\gamma n)\times r}$
where each entry is an independent $\mathcal{N}(0,1)$ random
variable, and then set $L^*:=A\times
B^\top$ (the ``clean'' part of $M$). Outliers, $C^*\in \mathbb{R}^{\gamma n\times
p}$ are generated either {\em neutrally}, where each entry of $C^*$
is {\it iid} $\mathcal{N}(0,1)$, or {\em adversarially}, where every column
is an identical copy of a random Gaussian vector. Outlier Pursuit succeeds
if $\hat{C} \in \PI$, and $\hat{L}\in \PP_U$.

\begin{figure}[htb!]
\begin{center}
\begin{tabular}{ccc}
 (a) Random Outlier & (b) Identical Outlier & (c) Noisy Outlier Detection\\
  \includegraphics[height=4.7cm, width=0.3\linewidth]{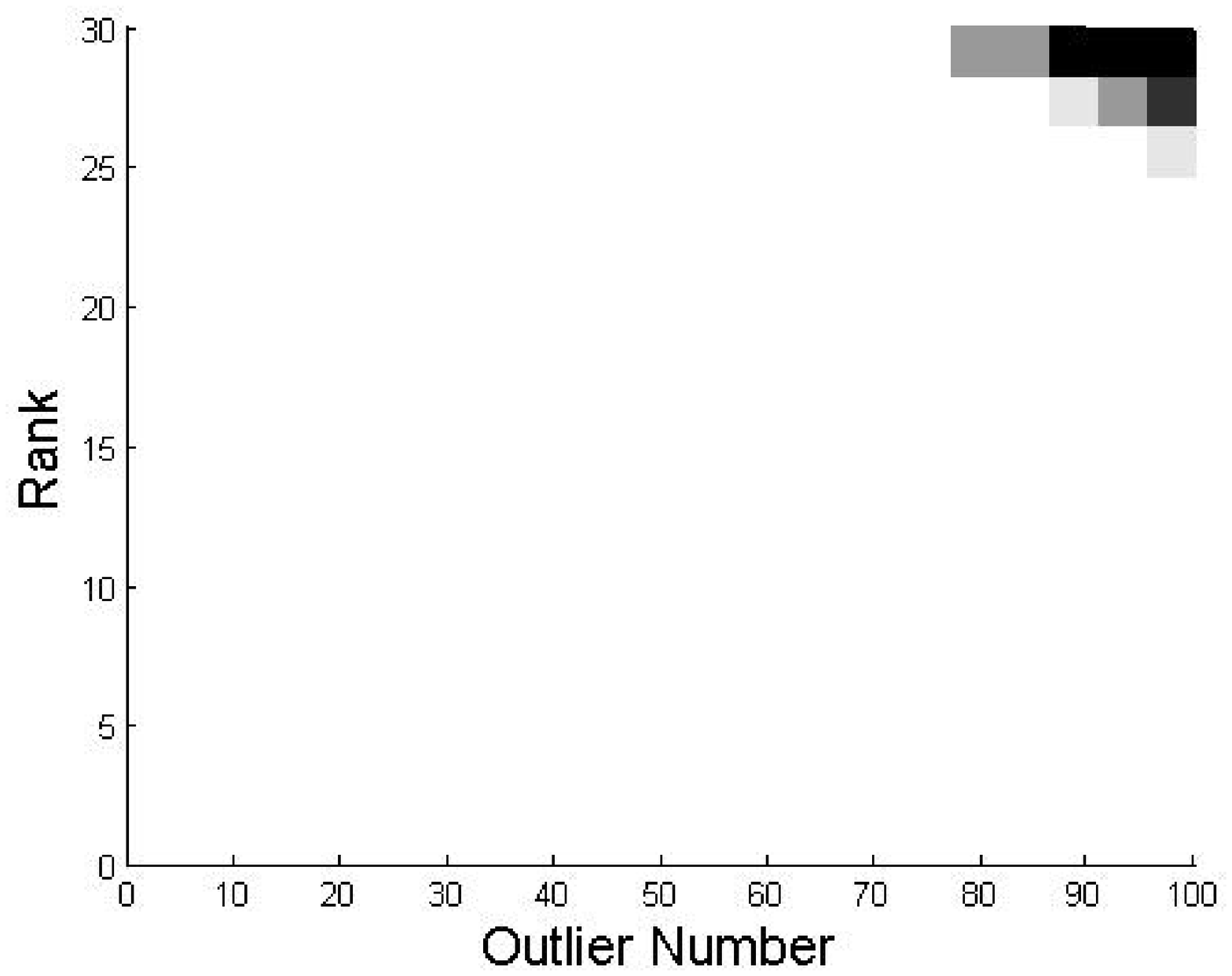}
   & \includegraphics[height=4.7cm, width=0.3\linewidth]{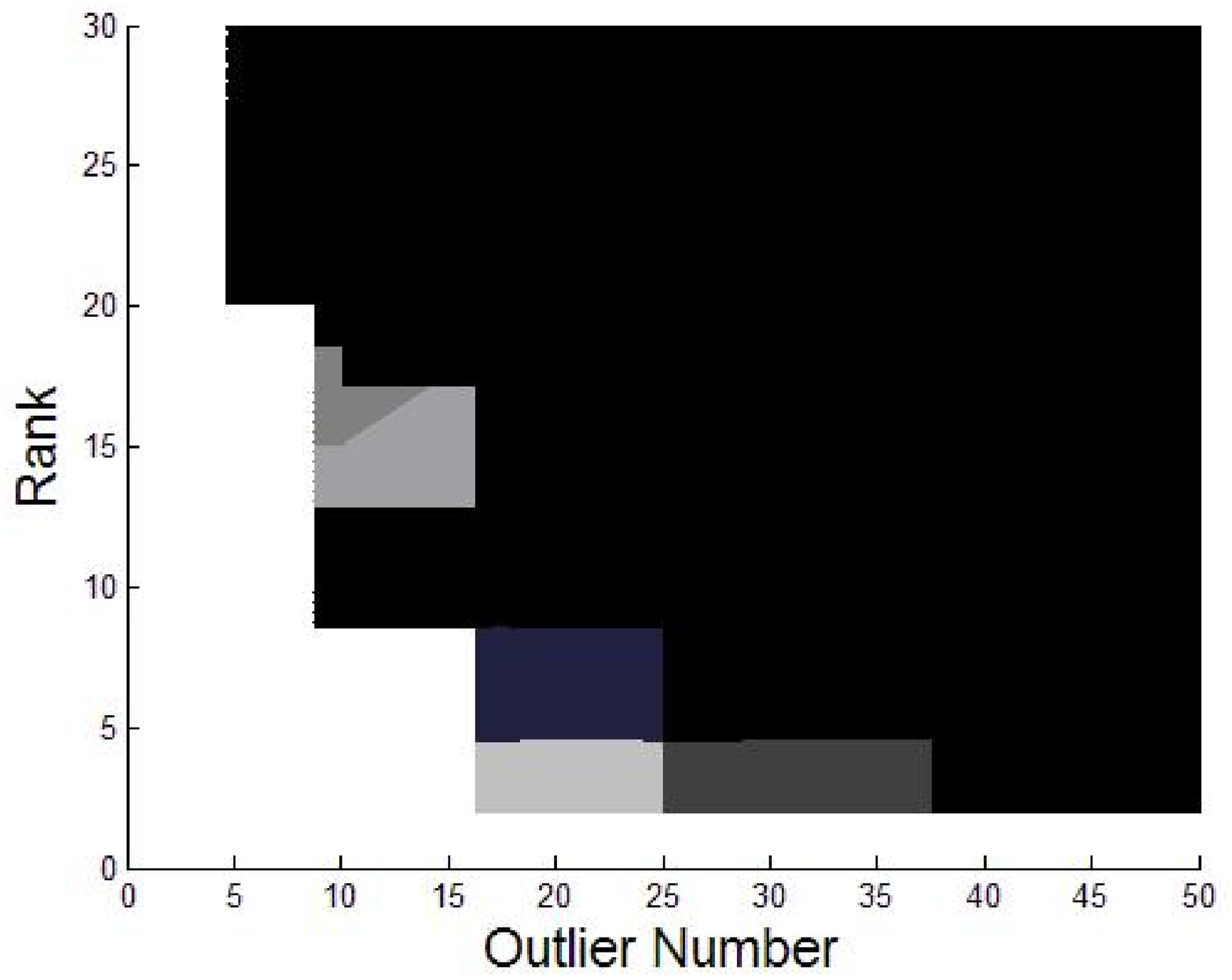} &
  \includegraphics[height=4.7cm, width=0.3\linewidth]{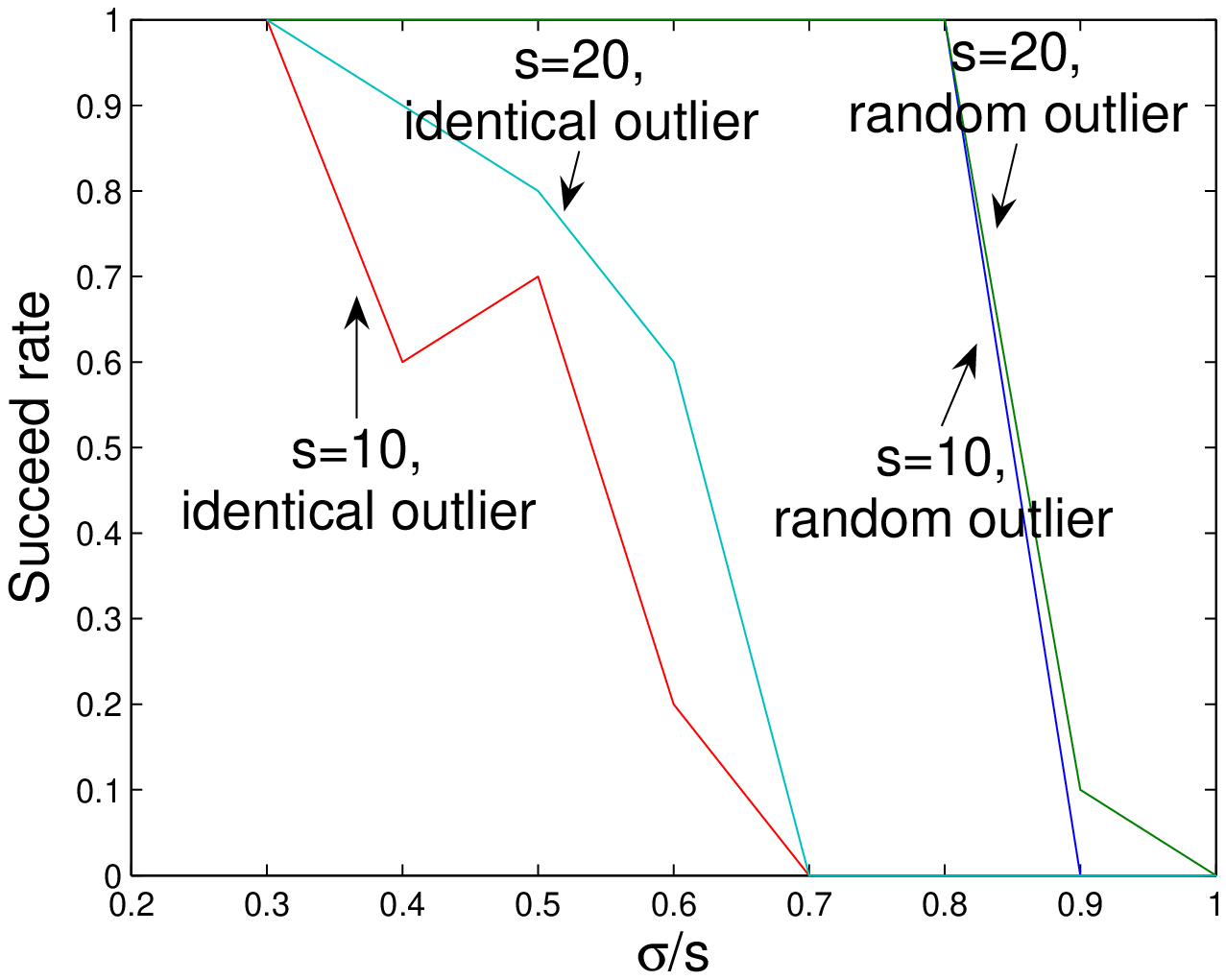}
\end{tabular} \caption{This figure shows the performance of our algorithm in the case of complete observation (compare the next figure). The results shown represent an average over 10 trials.} \label{fig.pt1}
\end{center}
\end{figure}
Figure~\ref{fig.pt1} shows the phase
transition property. We represent success in gray scale, with white denoting success, and black failure. 
When outliers are random (easier case) Outlier Pursuit succeeds even when $r=20$ with 100 outliers.
In the adversarial case, Outlier Pursuit succeeds when $r \times \gamma \leq c$,
and fails otherwise, consistent with our theory's predictions.
We then fix $r = \gamma n = 5$ and examine the outlier identification ability of Outlier
Pursuit with noisy observations. We scale each outlier so that the $\ell_2$ distance of the outlier
to the span of true samples equals a pre-determined value $s$. Each true
sample is thus corrupted with a Gaussian random vector with an
$\ell_2$ magnitude $\sigma$. We perform (noiseless) Outlier Pursuit
on this noisy observation matrix, and claim that the algorithm
successfully identifies outliers if for the resulting $\hat{C}$
matrix, $\|\hat{C}_j\|_2 < \|\hat{C}_i\|_2$ for all $j\not\in
\mathcal{I}$ and $i\in \mathcal{I}$, i.e., there exists a threshold
value to separate out outliers. Figure~\ref{fig.pt1} (c) shows the
result: when $\sigma/s\leq 0.3$ for the identical outlier case, and
$\sigma/s\leq 0.7$ for the random outlier case, Outlier Pursuit
correctly identifies the outliers.

We further study the case of decomposing $M$ under incomplete
observation, which is motivated by {\em robust collaborative
filtering}: we generate $M$ as before, but only observe each entry
with a given probability (independently). Letting $\Omega$ be the
set of observed entries, we solve
\begin{equation}\label{equ.incomplete}\mbox{Minimize:}\quad \|L\|_*+\lambda \|C\|_{1,2};\quad
\mbox{Subject to:}\quad
\PP_{\Omega}(L+C)=\PP_{\Omega}(M).\end{equation} The same success
condition is used. Figure~\ref{fig.pt2} shows a very promising
result:  the successful decomposition rate under incomplete
observation is close the the complete observation case even only
$30\%$ of entries are observed. Given this empirical result, a
natural direction of future research is  to understand theoretical
guarantee of~(\ref{equ.incomplete}) in the incomplete observation
case.
\begin{figure}[htb!]
\begin{center}
\begin{tabular}{ccc}
 (a) $30\%$ entries observed & (b) $80\%$ entries observed & (c) success rate {\em vs} observation ratio \\
  \includegraphics[height=4.7cm, width=0.3\linewidth]{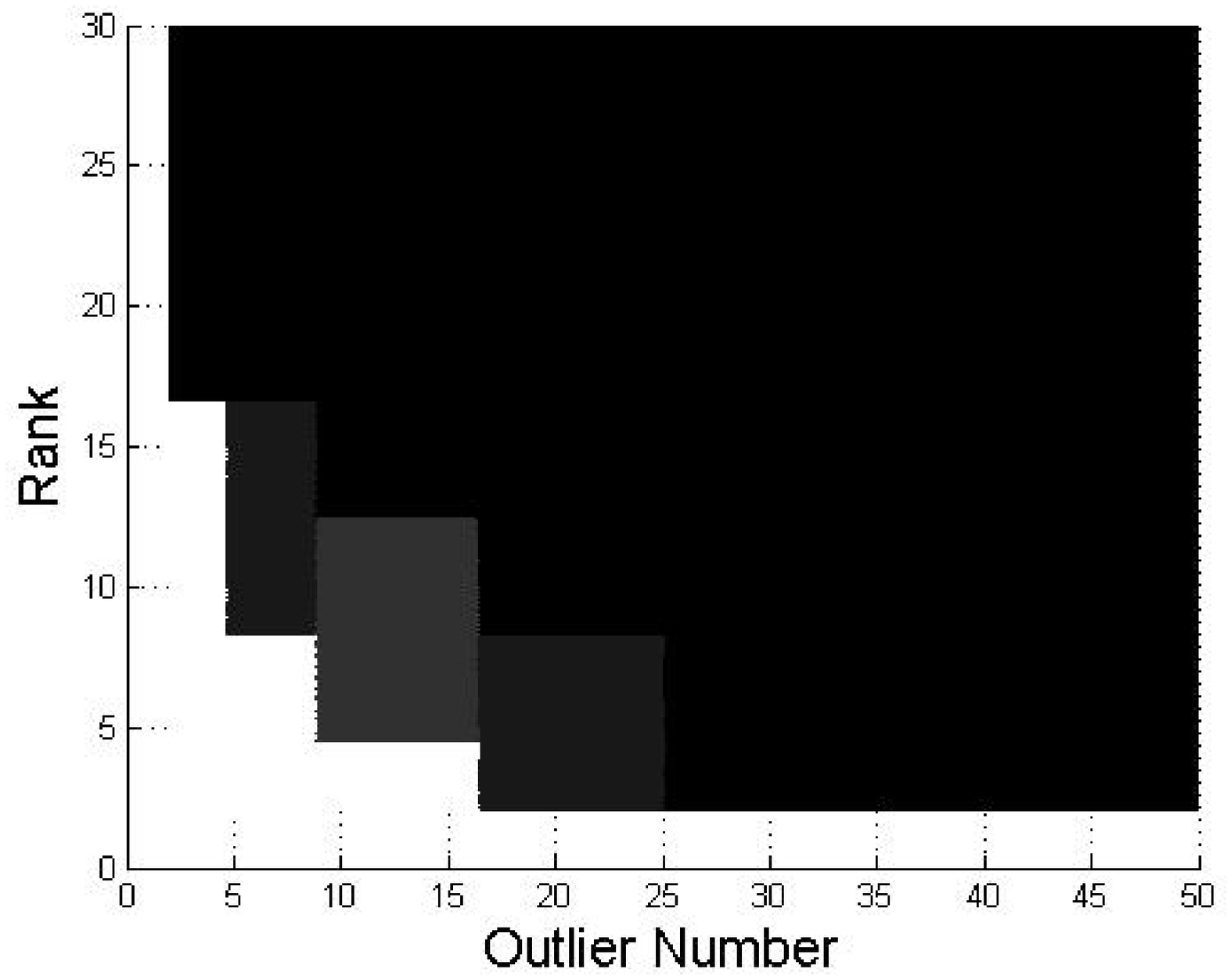}
    & \includegraphics[height=4.7cm, width=0.3\linewidth]{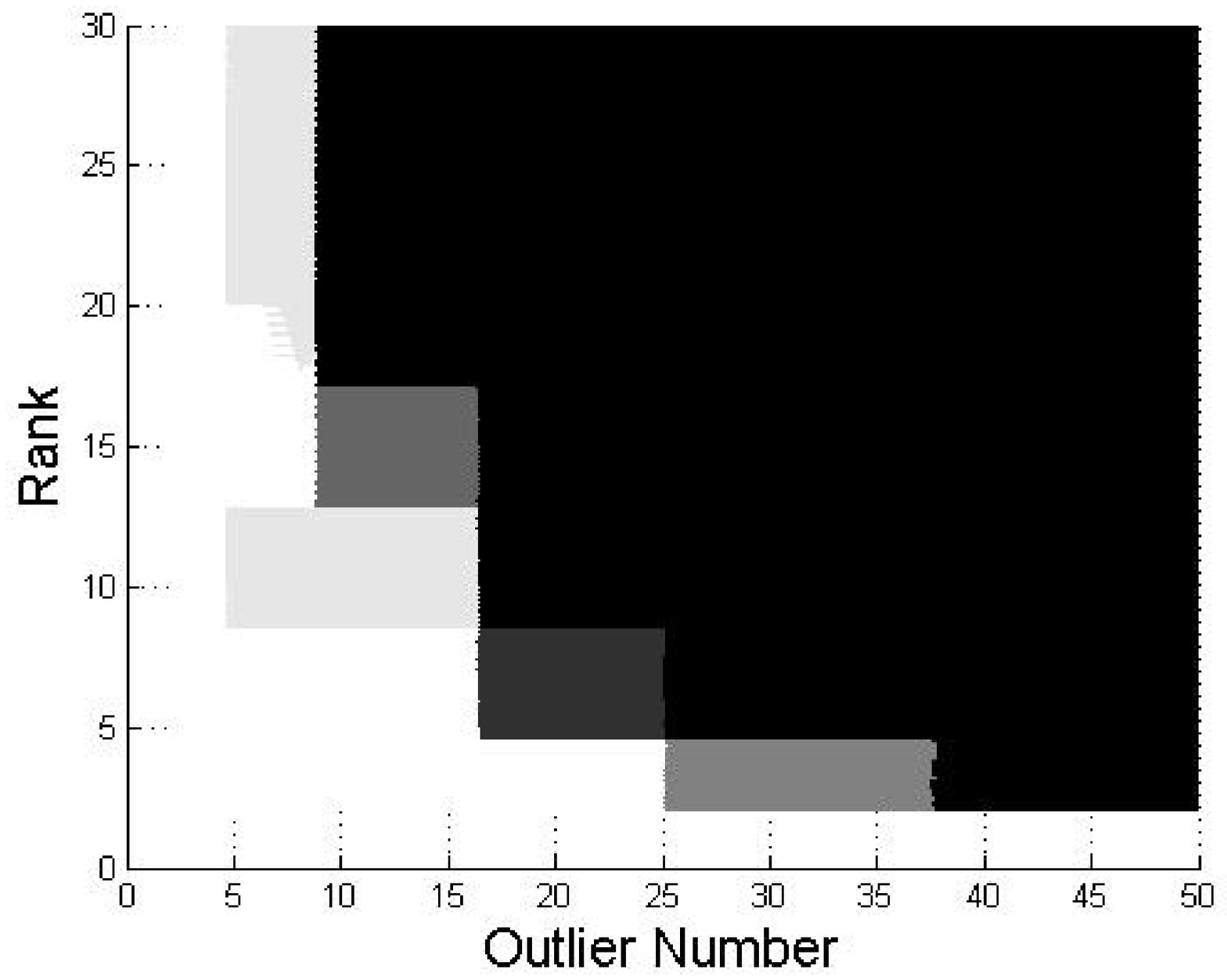} &
    \includegraphics[height=4.7cm, width=0.3\linewidth]{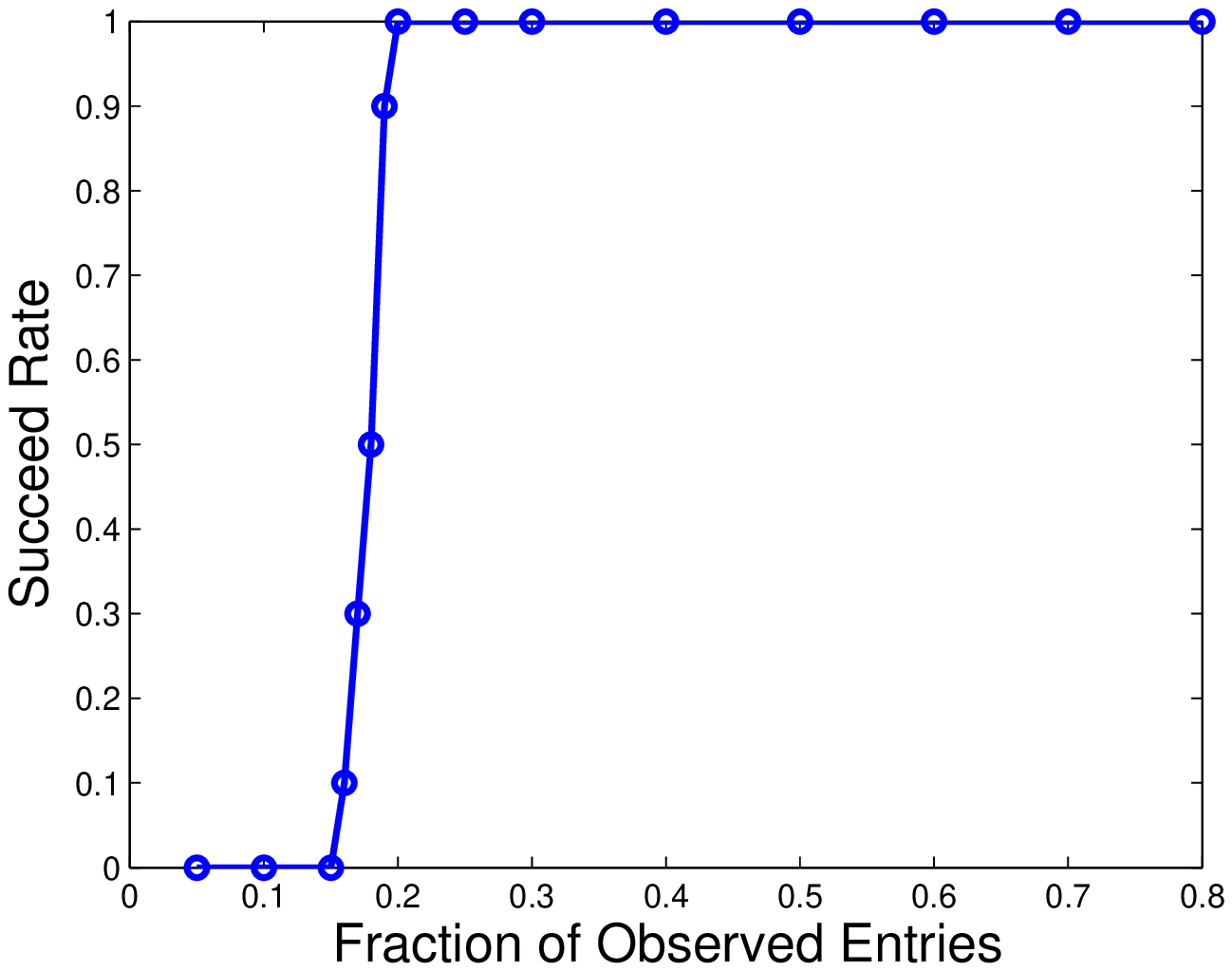}\\
\end{tabular} \caption{This figure shows the case of partial observation, where only a fraction of the entries, sampled uniformly at random, are observed.}\label{fig.pt2}
\end{center}
\end{figure}

Next we report some experimental results on the USPS digit data-set.
The goal of this experiment is to show that Outlier Pursuit can be
used to identify anomalies within the dataset. We use the data from
\cite{RasmussenWilliams06}, and construct the observation matrix $M$
as containing the first $220$ samples of digit ``1'' and the last
$11$ samples of ``7''. The learning objective is to correctly
identify all the ``7's''. Note that throughout the experiment, label
 information is unavailable to the algorithm, i.e., there is  no training stage. Since the
 columns  of  digit ``1'' are not exactly low rank,  an exact
 decomposition is not possible.  Hence, we use the $\ell_2$ norm of each column in the resulting $C$
 matrix to identify the outliers: a larger
 $\ell_2$ norm means that  the sample is more likely to be an
 outlier --- essentially, we apply thresholding after $C$ is obtained.
  Figure~\ref{fig.digit1}(a) shows the
 $\ell_2$ norm of each column of the resulting $C$ matrix.   We see that all ``7's'' are indeed identified. However, two
 ``1'' samples (columns $71$ and $137$) are also identified as
 outliers, due to the fact that these two samples are written in a
 way that is different from the rest of the ``1's'' as shown in
 Figure~\ref{fig.digit_detail}. Under the same setup, we also simulate the case where only $80\%$ of entries are
 observed. As Figure~\ref{fig.digit1} (b) and (c) show, similar
 results as that of the complete observation case are obtained, i.e., all true ``7's'' and also ``1's'' No 71, No 177 are identified.
 \begin{figure}[htb!]
\begin{center}
\begin{tabular}{lll}
(a) Complete Observation & (b) Partial
 Obs. (one run) & (c) Partial  Obs. (average)\\
  \includegraphics[height=4.7cm,  width=0.3\linewidth]{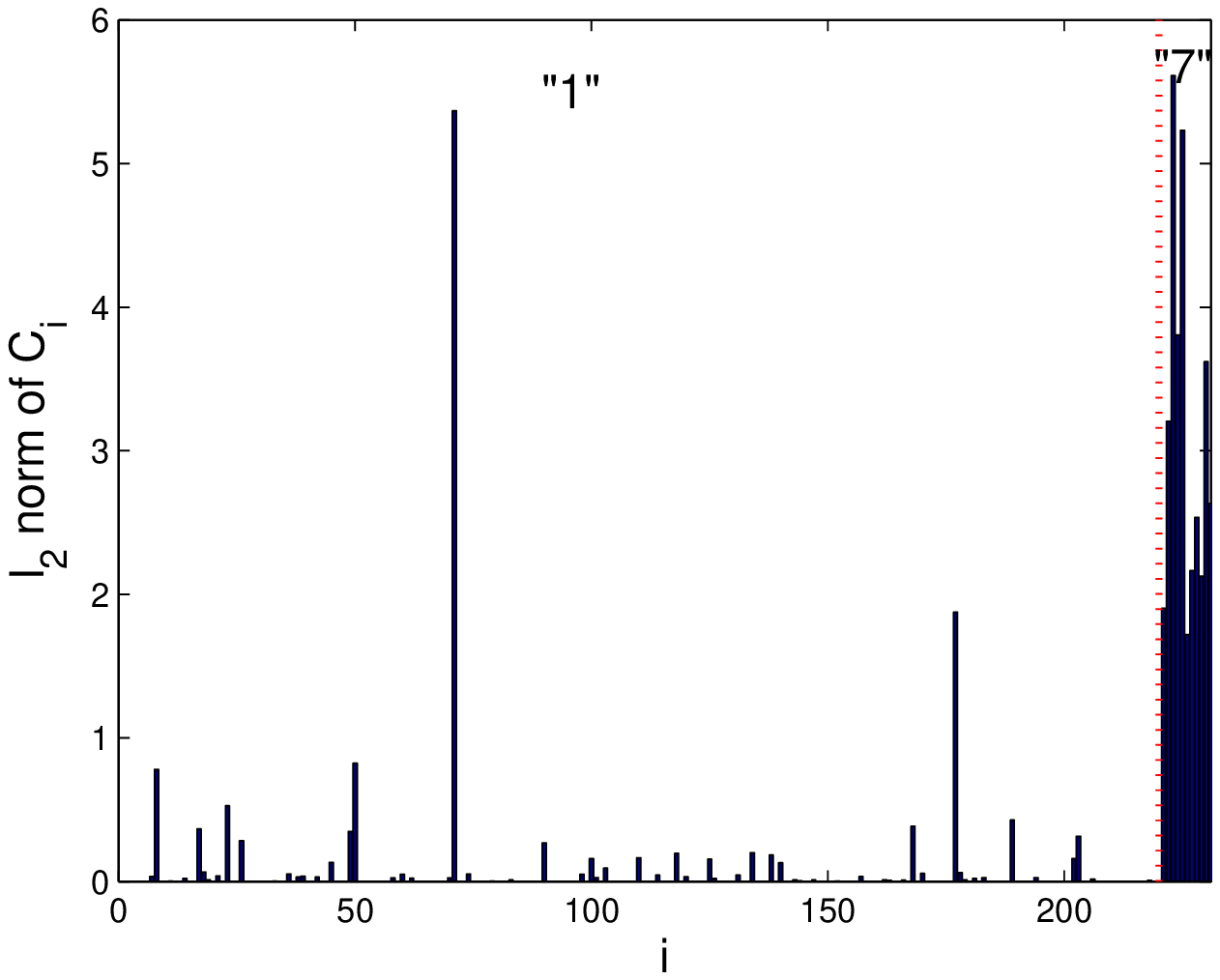} &
    \includegraphics[height=4.7cm,width=0.3\linewidth]{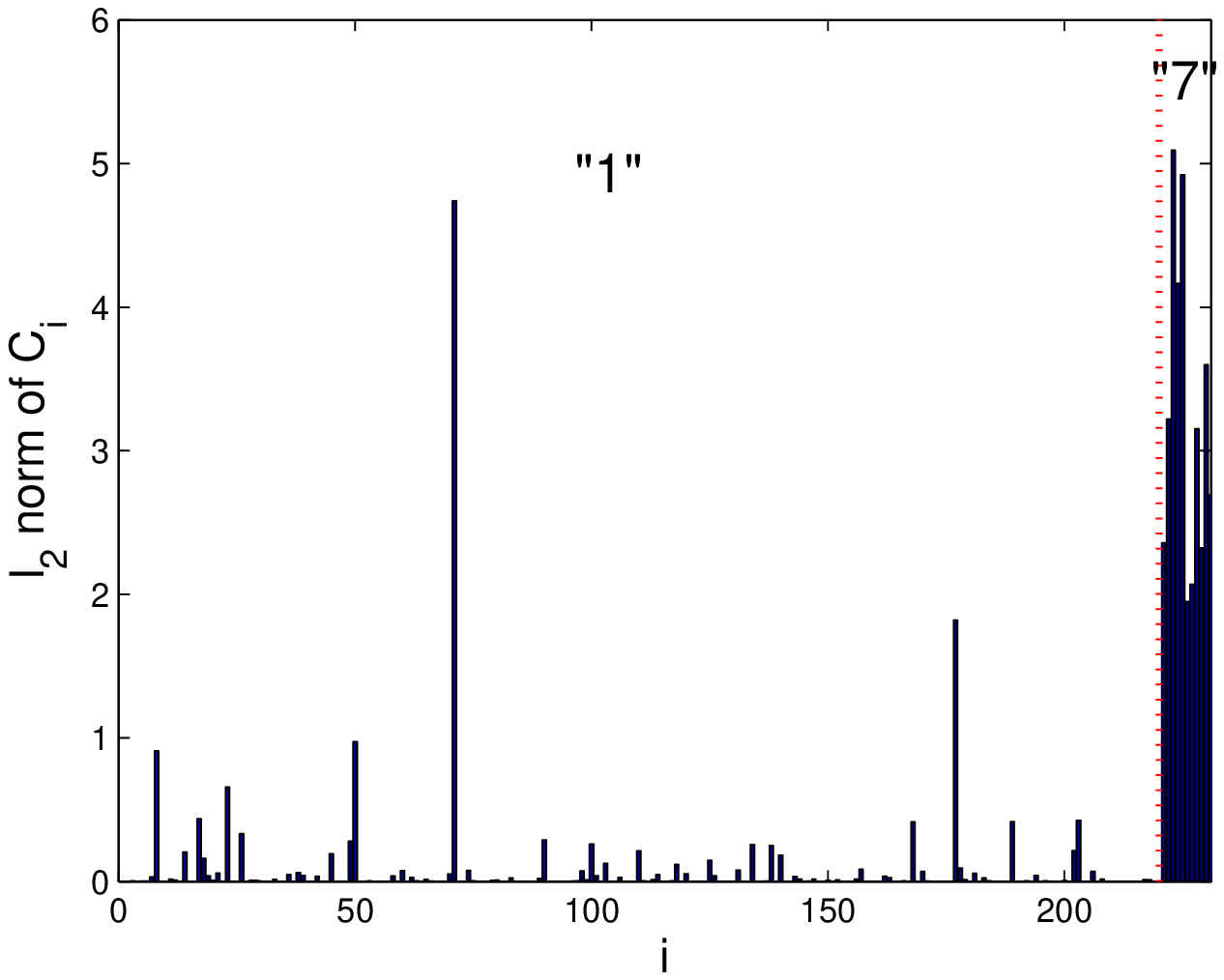}
  &    \includegraphics[height=4.7cm, width=0.3\linewidth]{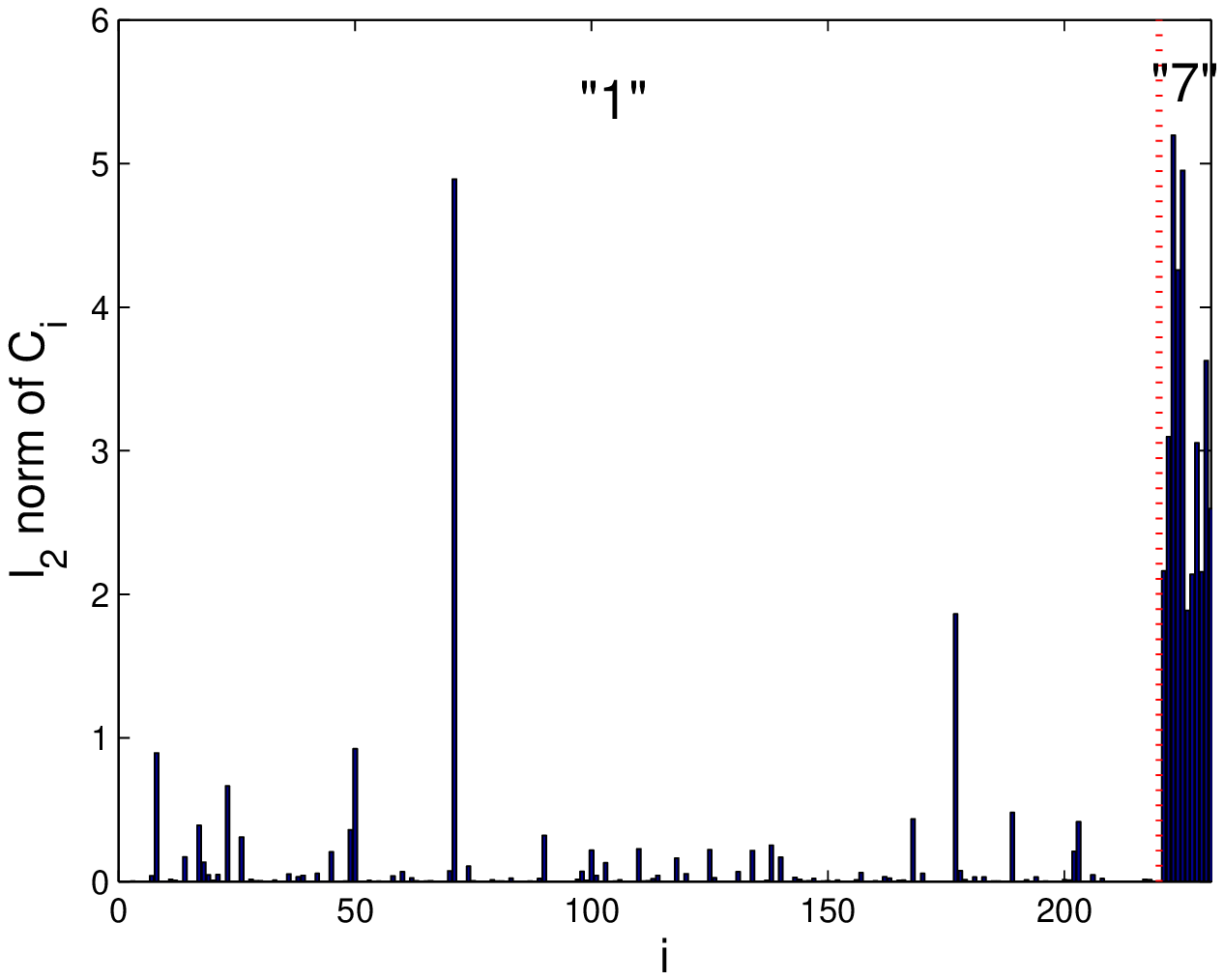}
\end{tabular} \caption{This figure shows the $\ell_2$ norm of each of the 220 columns of $C$. Large norm indicates that the algorithm believes that column is an outlier. All ``7's'' and two ``1's'' are identified as outliers.} \label{fig.digit1}
\end{center}
\end{figure}

 \begin{figure}[htb!]
\begin{center}
\begin{tabular}{cccccccccc}
  \includegraphics[height=1.8cm, width=0.1\linewidth]{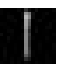} &
    \includegraphics[height=1.8cm, width=0.1\linewidth]{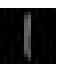} &
      \includegraphics[height=1.8cm, width=0.1\linewidth]{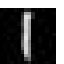} &
      &
        \includegraphics[height=1.8cm, width=0.1\linewidth]{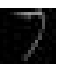} &
          \includegraphics[height=1.8cm, width=0.1\linewidth]{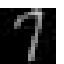} &
            \includegraphics[height=1.8cm, width=0.1\linewidth]{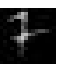}
            & &
  \includegraphics[height=1.8cm, width=0.1\linewidth]{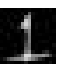}
   & \includegraphics[height=1.8cm, width=0.1\linewidth]{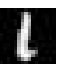} \\
   & ``1'' & & &  &``7'' & & & No 71 & No 177
\end{tabular} \caption{This figure shows the typical ``1's'', the typical ``7's'' and also the two abnormal ``1's'' identified by the algorithm as outliers.} \label{fig.digit_detail}
\end{center}
\end{figure}

\section{Conclusion and Future Direction}
This paper considers robust PCA from a matrix  decomposition
approach, and develops the Outlier Pursuit algorithm. Under some
mild conditions that are quite natural in most PCA settings, we show that Outlier Pursuit can exactly recover
the column support, and exactly identify outliers. This result is
new, differing both from results in Robust PCA, and also from
results using nuclear-norm approaches for matrix completion and
matrix reconstruction. One central innovation we introduce is the
use of an oracle problem. Whenever the recovery concept (in this
case, column space) does not uniquely correspond to a single matrix
(we believe many, if not most cases of interest, fit this description), the use of such a tool will be quite useful.
Immediate goals for future work include considering specific
applications, in particular, robust collaborative filtering (here,
the goal is to decompose a partially observed column-corrupted
matrix) and also obtaining tight bounds for outlier identification
in the noisy case.

{\small
\bibliographystyle{unsrt}
\bibliography{Phd1}}
\appendices

\section{Orthogonal Case}
\label{app:orthogonal}
This section investigates the special case where each outlier is
orthogonal to the span of true samples, as stated in the following
assumption.
\begin{assumption}\label{ass.orthogonal}
For $i\in \mathcal{I}_0$, $j\not\in \mathcal{I}_0$, we have
$M_i^\top M_j=0$.
\end{assumption}

In the orthogonal case, we are able to derive a {\em necessary and
sufficient} condition of Outlier Pursuit to succeed. Such condition
is of course a necessary condition for Outlier Pursuit to succeed in
the more general (non-orthogonal) case. Let
\[H_0=\left\{\begin{array}{cc} \frac{(C_0)_i}{\|(C_0)_i\|_2},
&\mbox{if}\,\, i\in \mathcal{I}_0;\\ 0
&\mbox{otherwise.}\end{array}\right.\]

\begin{theorem}\label{thm.orthogonal}Under
Assumption~\ref{ass.orthogonal}, Outlier Pursuit succeeds {\em if
and only if}
\begin{equation}\label{equ.nscondition}
\begin{split}
&\|H_0\| \leq 1/\lambda;\quad\|U_0 V_0^{\top}\|_{\infty,2}\leq
\lambda.
\end{split}\end{equation}
If both inequalities hold strictly, then Outlier Pursuit strictly
succeeds.
\end{theorem}

\begin{corollary}\label{cor.orthogonal} If the outliers are generated adversarial, and Assumption~\ref{ass.orthogonal} holds, then  Outlier Pursuit succeeds (for some $\lambda^*$) if and
only if
\[\frac{\gamma}{1-\gamma}\leq \frac{1}{\mu r}.\] Specifically, we can choose
$\lambda^*=\sqrt{\frac{\mu r+1}{n}}$.
\end{corollary}

\subsection{Proof of Theorem~\ref{thm.orthogonal}} The proof
consists of three steps. We first show that if Outlier Pursuit
succeeds, then $(L_0, C_0)$ must be an optimal solution to Outlier
Pursuit. Then using subgradient condition of optimal solutions to
convex programming, we show that the necessary and sufficient
condition for $(L_0, C_0)$ being optimal solution is the existence
of a dual certificate $Q$. Finally, we show that the existence of
$Q$ is equivalent to Condition~(\ref{equ.nscondition}) holds. We
devote a subsection for each step.
 \subsubsection{Step 1}
We need a technical lemma first.
 \begin{lemma}\label{lem.orth}Given $A \in \R^{m\times n}$,
we have \[ \|\mathcal{P}_{\mathcal{I}_0^c}(A)\|_* \leq \| A \|_*.\]
\end{lemma}
\begin{proof}
Fix $r\geq \rank(A)$. It is known that $\|A\|_*$ has the following
variational form (Lemma 5.1 of~\cite{RechtFazelParrilo2010}):
\begin{equation}\label{equ.var}
\begin{split}\|A\|_* \quad =\quad  &\mbox{Minimize:}_{X\in \R^{m\times r},
Y\in \R^{n\times
r}}\quad \frac{1}{2} (\|X\|^2_F +\|Y\|^2_F)\\
&\mbox{Subject to:}\quad X Y^\top =A.
\end{split}
\end{equation}
Note that for any $X Y^\top =A$,  we have \[ X \overline{Y}^\top =X
(\mathcal{P}_{\mathcal{I}_0^c} (Y^\top)) =
\mathcal{P}_{\mathcal{I}^c}(A),\] where $\overline{Y}$ is the matrix
resulted by setting all {\em rows} of $Y$ in $\mathcal{I}$ to zero.
Thus, by variational form of
$\|\mathcal{P}_{\mathcal{I}_0^c}(A)\|_*$, and note that
$\rank(\mathcal{P}_{\mathcal{I}_0^c}(A))\leq r$, we have
\[\|\mathcal{P}_{\mathcal{I}_0^c}(A)\|_* \leq \frac{1}{2}[\|X\|^2_F + \|\overline{Y}\|_F^2] \leq  \frac{1}{2}[\|X\|^2_F + \|Y\|_F^2].\]
Note this holds for any $X, Y$ such that $XY^\top =A$, the lemma
follows from~(\ref{equ.var}).
\end{proof}

\begin{theorem}\label{thm.step1}Under Assumption~\ref{ass.orthogonal}, for any $L'$, $C'$ such that
$L'+C'=M$, $\PP_{\mathcal{I}_0}(C')=C'$, and $\PP_{U_0} (L')=L'$, we
have
\[\|L_0\|_* +\lambda \|C_0\|_{1,2} \leq \|L'\|_* +\lambda \|C'\|_{1,2},\] with
the equality holds only when $L'=L_0$ and $C'=C_0$.
\end{theorem}
\begin{proof}Write $L'=L_0+\Delta$ and $C'=C_0-\Delta$. Since
$\PP_{U_0}(L')=L'$,
 we have that for $i\in \mathcal{I}_0$, $\PP_{U_0}\Delta_i=\Delta_i$,
 which implies that for $i\in \mathcal{I}_0$
 \[C_{0i}^{\top} \Delta_i = (C_{0i}^{\top}
 U) U^{\top} \Delta_i =0 \times U^{\top} \Delta_i,\]
 where the last equality holds from Assumption~\ref{ass.orthogonal}
 and the definition of $C_0$ (recall that $C_{0i}$ is the $i^{th}$ column of $C_0$).
 Thus, $\|C_0\|_{1,2}=\sum_{i\in \mathcal{I}} \|C_{0i}\|_2 \leq \sum_{i\in \mathcal{I}_0} \|C_{0i}+\Delta_i\|_2
 \leq \sum_{i=1}^n \|C_{0i}+\Delta_i\|_2= \|C'\|_{1,2}$, with equality only
holds when $\Delta=0$.

 Further note that $\PP_{\mathcal{I}_0}(C')=C'$ implies that
 $\PP_{\mathcal{I}_0}(\Delta)=\Delta$, which by definition of $L_0$ leads to
 \[L_0=\PP_{\mathcal{I}_0^c} L'.\]
Thus, Lemma~\ref{lem.orth} implies $\|L_0\|_*\leq \|L'\|_*$. The
theorem thus follows.
\end{proof}
Note that Theorem~\ref{thm.step1} essentially says that
in the orthogonal case,  if Outlier Pursuit succeeds, i.e., it
outputs a pair $(L', C')$ such that $L'$ has the correct column
space, and $C'$ has the correct column support, then $(L_0, C_0)$
must be the output. This makes it possible to restrict out attention
to investigate when the solution to Outlier Pursuit is $(L_0,
C_0)$.
\subsubsection{Step 2}

\begin{theorem}\label{thm.step2}Under
Assumption~\ref{ass.orthogonal}, $(L_0, C_0)$ is an optimal solution
to Outlier Pursuit if and only if
 there exists $Q$ such that
\begin{equation}\label{equ.condition}\begin{split}
(a)\quad &\mathcal{P}_{T_0}(Q) =U_0V_0^{\top};\\
(b)\quad &\|\mathcal{P}_{T_0^\bot}(Q)\| \leq
1;\\
(c)\quad&\mathcal{P}_{\mathcal{I}_0}(Q) = \lambda
H_0;\\(d)\quad&\|\mathcal{P}_{\mathcal{I}_0^c}(Q)\|_{\infty,2}\leq
\lambda.
\end{split}\end{equation}
Here $\PP_{T_0}(\cdot)\triangleq \PP_{T(L_0)}(\cdot)$. In addition,
if both inequalities are strict, then $(L_0, C_0)$ is the unique
optimal solution.
\end{theorem}
\begin{proof}Standard convex analysis yields that $(L_0, C_0)$ is an
optimal solution to Outlier Pursuit if and only if there exists a
dual matrix $Q$ such that
\[Q\in \partial \|L_0\|_*;\quad Q\in \partial \lambda\|C_0\|_{1, 2}.\]
Note that a matrix $Q$ is a subgradient of $\|\cdot\|_*$ evaluated
at $L_0$ if and only if it satisfies
\[ \mathcal{P}_{T_0}(Q) =U_0V_0^{\top};\quad\mbox{and}\quad
\|\mathcal{P}_{T_0^\bot}(Q)\| \leq 1.\]

Similarly, $Q$ is a subgradient of $\lambda\|\cdot\|_{1, 2}$
evaluated at $C_0$ if and only if
\[ \mathcal{P}_{\mathcal{I}_0}(Q) = \lambda H_0;\quad \mbox{and}\quad
\|\mathcal{P}_{\mathcal{I}_0^c}(Q)\|_{\infty,2}\leq  \lambda.\]
Thus, we conclude the proof of the first part of the theorem, i.e., the necessary
and sufficient condition of $(L_0, C_0)$ being an optimal solution.

Next we show that if both inequalities are strict, then $(L_0, C_0)$
is the unique optimal solution.
 Fix $\Delta\not=0$, we show that $(L_0+\Delta,
C_0-\Delta)$ is strictly worse than $(L_0, C_0)$. Let $W$ be such
that $\|W\|=1$, $\langle W,
\PP_{T_0^\bot}(\Delta)\rangle=\|\PP_{T_0^\bot} \Delta\|_{*}$, and
$\PP_{T_0} W=0$. Let $F$ be such that such that
\[F_i=\left\{\begin{array}{ll} \frac{-\Delta_i}{\|\Delta_i\|_2} &
\mbox{if}\,\, i\not\in \mathcal{I}_0,\,\, \mbox{and}\,\, \Delta_i \not=0\\
0 & \mbox{otherwise.}\end{array}\right.\] Then $U_0 V_0^{\top}+W$ is
a subgradient of $\|\cdot\|_*$ at $L_0$ and $H_0+F$ is a subgradient
of $\|\cdot\|_{1,2}$ at $C_0$.
 Then we have
\begin{equation*}\begin{split}
&\|L_0+\Delta\|_*+\lambda\|C_0-\Delta\|_{1,2}\\
& \geq \|L_0\|_*+\lambda\|C_0\|_{1,2}+<U_0 V_0^{\top}+W, \Delta>
-\lambda <H_0+F,
\Delta>\\
&=\|L_0\|_*+\lambda\|C_0\|_{1,2}+
\|\mathcal{P}_{T_0^\bot}(\Delta)\|_*+\lambda
\|\mathcal{P}_{\mathcal{I}_0^c}(\Delta)\|_{1,2}+ <U_0 V_0^{\top}
-\lambda H_0,
\Delta>\\
&=\|L_0\|_*+\lambda\|C_0\|_{1,2}+
\|\mathcal{P}_{T_0^\bot}(\Delta)\|_*+\lambda
\|\mathcal{P}_{\mathcal{I}_0^c}(\Delta)\|_{1,2}+
<Q-\mathcal{P}_{T_0^\bot}(Q) -(Q-\mathcal{P}_{\mathcal{I}_0^c}(Q)),
\Delta>
\\
&=\|L_0\|_*+\lambda\|C_0\|_{1,2}+
\|\mathcal{P}_{T_0^\bot}(\Delta)\|_*+\lambda
\|\mathcal{P}_{\mathcal{I}_0^c}(\Delta)\|_{1,2}+
<-\mathcal{P}_{T_0^\bot}(Q), \Delta>+<
\mathcal{P}_{\mathcal{I}_0^c}(Q),
\Delta>\\
&\geq
\|L_0\|_*+\lambda\|C_0\|_{1,2}+(1-\|\mathcal{P}_{T_0^\bot}(Q)\|)
\|\mathcal{P}_{T_0^\bot}(\Delta)\|_*+(\lambda-\|\mathcal{P}_{\mathcal{I}_0^c}(Q)\|_{\infty,
2})
\|\mathcal{P}_{\mathcal{I}_0^c}(\Delta)\|_{1,2}\\
&\geq \|L_0\|_*+\lambda\|C_0\|_{1,2},\end{split}\end{equation*}where
the last inequality is strict unless
\begin{equation}\label{equ.inspace}
\|\mathcal{P}_{T_0^\bot}(\Delta)\|_*=\|\mathcal{P}_{\mathcal{I}_0^c}(\Delta)\|_{1,2}=0.
\end{equation}
We next show that Condition~(\ref{equ.inspace}) also implies a
strict increase of the objective function to complete the proof.
Note that Equation~(\ref{equ.inspace}) is equivalent to
$\Delta=\mathcal{P}_{T_0}(\Delta)=\mathcal{P}_{\mathcal{I}_0}(\Delta)$,
and note that
\begin{equation*}\begin{split}
&\mathcal{P}_{U_0}(\Delta)
=\mathcal{P}_{T_0}(\Delta)-\mathcal{P}_{V_0}(\Delta)+\mathcal{P}_{U_0
}\mathcal{P}_{V_0}(\Delta)= \Delta - (I-\mathcal{P}_{U_0})
\mathcal{P}_{V_0}\Delta.
\end{split}\end{equation*}
Since $\PP_{\mathcal{I}_0}(V_0^{\top})=0$,
$\PP_{\mathcal{I}_0}(\Delta)=\Delta$ implies that
$\mathcal{P}_{V_0}(\Delta)=0$, which means
\[\Delta=\mathcal{P}_{U_0}(\Delta)=\mathcal{P}_{\mathcal{I}_0}(\Delta).\]
Thus, $\PP_{U_0}(L_0+\Delta)=L_0+\Delta$, and
$\PP_{\mathcal{I}_0}(C_0-\Delta)=C_0-\Delta$. By
Theorem~\ref{thm.step1},
$\|L_0+\Delta\|_*+\lambda\|C_0-\Delta\|_{1,2}
> \|L_0\|_*+\lambda\|C_0\|_{1,2}$, which completes the proof.
\end{proof}

\subsubsection{Step 3}
\begin{theorem}\label{thm.step3}Under Assumption~\ref{ass.orthogonal},
if there exists any matrix $Q$ that satisfies
Condition~(\ref{equ.nscondition}), then $U_0V_0^{\top}+\lambda H_0$
satisfies~(\ref{equ.nscondition}).
\end{theorem}
\begin{proof}
Denote $Q_0\triangleq U_0 V_0^{\top} +\lambda H_0$. We first show
that the two equalities of Condition~(\ref{equ.nscondition}) hold.
Note that
\[\PP_{T_0}(Q_0)=\PP_{T_0}(U_0 V_0^{\top})+
\lambda\PP_{T_0}(H_0)=U_0 V_0^{\top}+\lambda
[\PP_{U_0}(H_0)+\PP_{V_0}(H_0)-\PP_{U_0}\PP_{V_0}(H_0)].\] Further
note that $\PP_{U_0}(H_0)=U_0(U_0^{\top} H_0)=0$ due to
Assumption~\ref{ass.orthogonal}, and $\PP_{V_0}(H_0)=0$ because
$\PP_{\mathcal{I}_0}H_0=H_0$ and $\PP_{\mathcal{I}_0}
(V_0^{\top})=0$ lead to $H_0V_0=0$. Hence
$$
\PP_{T_0}(Q_0)=U_0V_0^{\top}.
$$
Furthermore,
$$
\PP_{\mathcal{I}_0}(Q_0)=\PP_{\mathcal{I}_0}(U_0 V_0^{\top})+\lambda
\PP_{\mathcal{I}_0}(H_0)=U_0 \PP_{\mathcal{I}_0}(V_0^{\top})
+\lambda H_0=\lambda H_0.
$$
 Here, the last equality holds because
$\PP_{\mathcal{I}_0}(V_0^{\top})=0$. Note that this also implies
that
\begin{equation}\label{equ.proofinstep3}
\PP_{T_0^\bot}(H_0)=H_0;\quad \PP_{\mathcal{I}_0^c}(U_0 V_0^{\top})
=U_0 V_0^{\top}.
\end{equation}

Now consider any matrix $Q$ that also satisfies the two equalities.
Let $Q=U_0 V_0^{\top}+\lambda H_0+\Delta$, note that $Q$ satisfies
$\PP_{\mathcal{I}_0}(Q)=\lambda H_0$ and $\PP_{\mathcal{T}_0}(Q)=U_0
V_0^{\top}$, which leads to
$$
\PP_{\mathcal{I}_0}(\Delta)=0;\quad
\mbox{and}\,\,\PP_{\mathcal{T}_0}(\Delta)=0.
$$
Thus,
$$
\PP_{\mathcal{I}_0^c}(Q)= U_0 V_0^{\top}+\Delta;\quad\mbox{and}\quad
\PP_{\mathcal{T}_0^\bot}(Q)=\lambda H_0 +\Delta.
$$
Note that
\begin{equation*}\begin{split}
&\|U_0 V_0^{\top}+\Delta\|_{\infty, 2}=\max_{i} \|U_0
(V_0^{\top})_i+\Delta_i\|_2\\ \geq  &\max_{i} \|U_0
(V_{0}^{\top})_i\|_2=\|U_0 V_0^{\top}\|_{\infty,2}.
\end{split}\end{equation*}
Here, the inequality holds because $\PP_{\mathcal{T}_0}(\Delta)=0$
implies that $\Delta_i$ are orthogonal to the span of $U$. Note that
the inequality is strict when $\Delta\not=0$.

On the other hand
\begin{equation*}
\begin{split}
&\|\lambda H_0\|=\max_{\|\mathbf{x}\|\leq 1,\|\mathbf{y}\|\leq 1}
\mathbf{x}^\top (\lambda H_0) \mathbf{y}
\stackrel{(a)}{=}\max_{\|\mathbf{x}\|\leq 1,\|\mathbf{y}\|\leq 1,
\PP_{\mathcal{I}_0^c}(\mathbf{y}^\top)=0}
\mathbf{x}^\top (\lambda H_0 ) \mathbf{y}\\
&\stackrel{(b)}{=}\max_{\|\mathbf{x}\|\leq 1,\|\mathbf{y}\|\leq 1,
\PP_{\mathcal{I}_0^c}(\mathbf{y}^\top)=0} \mathbf{x}^\top (\lambda
H_0 +\Delta) \mathbf{y}\leq \max_{\|\mathbf{x}\|\leq
1,\|\mathbf{y}\|\leq 1} \mathbf{x}^\top (\lambda H_0 +\Delta)
\mathbf{y}= \|\lambda H_0 +\Delta\|.
\end{split}
\end{equation*}
Here, (a) holds because $\PP_{\mathcal{I}_0}H_0=H_0$, thus for any
$\mathbf{y}$, set all $y_i=0$ for $i\not\in \mathcal{I}_0$ does not
change $\mathbf{x}^\top (\lambda H_0 ) \mathbf{y}$; while (b) holds
since $\mathcal{P}_{\mathcal{I}_0^c}\Delta= \Delta$.

Thus, if $Q$ satisfies the two inequalities, then so does $Q_0$,
which completes the proof.
\end{proof}

Note that by Equation~(\ref{equ.proofinstep3}) we have
$$
\PP_{T_0^\bot}(H_0)=H_0;\quad \PP_{\mathcal{I}_0^c}(U_0 V_0^{\top})
=U_0 V_0^{\top}.
$$
 Thus,
 Theorem~\ref{thm.step1}, Theorem~\ref{thm.step2} and
Theorem~\ref{thm.step3} together establish
Theorem~\ref{thm.orthogonal}.

\subsection{Proof of Corollary~\ref{cor.orthogonal}}

Corollary~\ref{cor.orthogonal} holds due to the following   lemma
that tightly bounds $\|H_0\|$ and $\|U_0 V_0^{\top}\|_{\infty, 2}$.
\begin{lemma}\label{lem.h} We have
(I)   $\|H_0\|\leq \sqrt{\gamma n}$, and the inequality is tight.
(II) $\|U_0 V_0^{\top}\|_{\infty, 2} =
\max_{i}\|V_0^{\top}\mathbf{e}_i\|_2=\sqrt{\frac{\mu
r}{(1-\gamma)n}}$.
\end{lemma}
\begin{proof} Following the variational form of the operator norm, we
have
\begin{equation*}\begin{split}
\|H_0\|&= \max_{\|\mathbf{x}\|_2 \leq 1, \|\mathbf{y}\|_2\leq 1}
\mathbf{x}^\top H_0 \mathbf{y}=\max_{\|\mathbf{x}\|_2\leq 1
}\|\mathbf{x}^\top H_0\|_2=\max_{\|\mathbf{x}\|_2\leq
1}\sqrt{\sum_{i=1}^n (\mathbf{x}^\top H_i)^2}\leq \sqrt{\sum_{i\in
\mathcal{I}_0} 1}=\sqrt{|\mathcal{I}_0|}=\sqrt{\gamma n}.
\end{split}\end{equation*}
The inequality holds because $\|(H_{0})_i\|_2=1$ when $i\in
\mathcal{I}_0$, and equals zero otherwise. Note that if we let
$(H_{0})_i$ all be the same, such as taking identical outliers, the
inequality is tight.

By definition we have $ \|U_0 V_0^{\top}\|_{\infty,2}=\max_i \| U_0
(V_0^{\top})_i\|_2 \stackrel{(a)}{=}\max_i \|(V_0^{\top})_i\|_2
=\max_i\|V_0^{\top}\mathbf{e}_i\|_2$. Here (a) holds since $U_0$ is
orthonormal.   The second claim hence follows from definition of
$\mu$.
\end{proof}

\end{document}